\documentclass[11pt]{article}
\usepackage{amsthm}
\usepackage{amsmath}
\usepackage{amssymb}
\usepackage{amsfonts}
\usepackage{mathrsfs}
\usepackage{bbm}
\usepackage{pifont,xspace,fullpage,epsfig, wrapfig}

\usepackage{algorithm}
\usepackage{enumitem}
\usepackage{color}
\definecolor{light-gray}{gray}{0.5}
\newcommand{\VC}{{\rm VC}}
\newcommand{\AAA}{\mathcal A}
\newcommand{\BBB}{\mathcal B}

\newcommand{\LLL}{\mathcal L}
\newcommand{\PPP}{\mathcal P}
\newcommand{\e}{\mathrm{e}}
\newcommand{\RepDim}{{\rm RepDim}}
\newcommand{\rectangle}{{\tt RECTANGLE}}
\newcommand{\db}{S}
\newcommand{\error}{{\rm error}}
\newcommand{\thresh}{{\tt THRESH}}
\newtheorem{thm}{Theorem}[section]
\newtheorem{obs}[thm]{Observation}
\newtheorem{claim}[thm]{Claim}
\newtheorem{lem}[thm]{Lemma}
\newtheorem{prop}[thm]{Proposition}
\newtheorem{cor}[thm]{Corollary}
\newtheorem{defn}[thm]{Definition}

\newtheorem{rem}[thm]{Remark}

\newcommand{\mynote}[2]{{\textcolor{#1}{ #2}}}
\newcommand{\gray}[1]{\mynote{light-gray}{{\footnotesize #1}}}
\newcommand{\ppp}{\mathsf p}
\newcommand{\hhh}{\mathsf h}
\newcommand{\remove}[1]{}

\begin{document}

\title{\Large Learning Privately with Labeled and Unlabeled Examples}
\author{Amos Beimel\thanks{Supported by 
a grant from the Israeli Science and Technology ministry,
by a Israel Science Foundation grant 544/13, and by the Frankel Center for Computer Science. }
\quad Kobbi Nissim\thanks{Work done while the second author was a visiting scholar at the Harvard Center for Research on Computation and Society (supported by NSF grant CNS-1237235) and at the Boston University Hariri Institute for Computing and Computational Science \& Engineering. Supported in part by Israel Science Foundation grant no.\ 276/12.}
\quad Uri Stemmer\thanks{
Supported by the Ministry of Science and Technology (Israel), by the Check Point Institute for Information Security, by the IBM PhD Fellowship Awards Program, and by the Frankel Center for Computer Science.} \\\\
Dept.\ of Computer Science \\
Ben-Gurion University of the Negev\\
{\tt \{beimel|kobbi|stemmer\}@cs.bgu.ac.il}}
\date{}

\maketitle

%\pagenumbering{arabic}
%\setcounter{page}{1}%Leave this line commented out.

\begin{abstract} %\small%\baselineskip=9pt 
A {\em private learner} is an algorithm that given a sample of labeled individual examples outputs a generalizing hypothesis while preserving the privacy of each individual. In 2008, Kasiviswanathan et al.\ (FOCS 2008) gave a generic construction of private learners, in which the sample complexity is (generally) higher than what is needed for non-private learners.
This gap in the sample complexity was then further studied in several followup papers, showing that (at least in some cases) this gap is unavoidable.
Moreover, those papers considered ways to overcome the gap, by relaxing either the privacy or the learning guarantees of the learner.\\
\indent We suggest an alternative approach, inspired by the (non-private) models of {\em semi-supervised learning} and {\em active-learning}, where the focus is on the sample complexity of {\em labeled} examples whereas {\em unlabeled} examples are of a significantly lower cost. We consider private semi-supervised learners that operate on a random sample, where only a (hopefully small) portion of this sample is labeled.
The learners have no control over which of the sample elements are labeled.
Our main result is that the labeled sample complexity of private learners is characterized by the VC dimension.\\
\indent We present two generic constructions of private semi-supervised learners. The first construction is of learners where the labeled sample complexity is proportional to the VC dimension of the concept class, however, the unlabeled sample complexity of the algorithm is as big as the representation length of domain elements. Our second construction presents a new technique for decreasing the labeled sample complexity of a given private learner, while roughly maintaining its unlabeled sample complexity. In addition, we show that in some settings the labeled sample complexity does not depend on the privacy parameters of the learner.
\end{abstract}

\section{Introduction}
A {\em private learner} is an algorithm that given a sample of labeled  examples, where each example represents an individual, outputs a generalizing hypothesis while preserving the privacy of each individual. This formal notion, combining the requirements of PAC learning~\cite{Valiant84} and Differential Privacy~\cite{DMNS06}, was presented in 2008 by Kasiviswanathan et al.~\cite{KLNRS08}, who also gave a generic construction of private learners. However, the sample complexity of the learner of~\cite{KLNRS08} is (generally) higher than what is needed for non-private learners. Namely, their construction requires $O(\log|C|)$ samples for learning a concept class $C$, as opposed to the non-private sample complexity of $\Theta(\VC(C))$.

This gap in the sample complexity was studied in several followup papers. For {\em pure} differential privacy, it was shown that in some cases this gap can be closed with the price of giving up proper learning -- where the output hypothesis should be from the learned concept class -- for {\em improper} learning. Indeed, it was shown that for the class of point functions over domain of size $2^d$, the sample complexity of every proper private learner is $\Omega(d)$ (matching the upper bound of~\cite{KLNRS08}), whereas there exist improper private learners with sample complexity $O(1)$ that use pseudorandom or pairwise independent functions as their output hypotheses~\cite{BBKN12, BNS13}.\footnote{To simplify the exposition, we omit in this section dependency on all variables except for $d$, corresponding to the representation length of domain elements.} 
A complete characterization for the sample complexity of pure-private improper-learners was given in~\cite{BNS13} in terms of a new dimension -- the Representation Dimension. They showed that $\Theta(\RepDim(C))$ examples are both necessary and sufficient for a pure-private improper-learner for a class $C$. Following that, Feldman and Xiao~\cite{FX14} separated the sample complexity of pure-private learners from that of non-private ones, and showed that the representation dimension can sometimes be significantly bigger then the VC dimension. For example, they showed that every pure-private learner (proper or improper) for the class of thresholds over $\{0,1\}^d$ requires $\Omega(d)$ samples~\cite{FX14} (while there exists a non-private proper-learner with sample complexity $O(1)$).

Another approach for reducing the sample complexity of private learners is to relax the privacy requirement to {\em approximate} differential privacy. This relaxation was shown to be significant as it allows privately and {\em properly} learning point functions with $O(1)$ sample complexity, and threshold functions with sample complexity $2^{O(\log^* d)}$~\cite{BNS13b}. 
Recently, Bun et al.~\cite{BNSV14} showed that the dependency in $\log^* d$ in necessary. Namely, they showed that every approximate-private proper-learner for the class of thresholds over $\{0,1\}^d$ requires $\Omega(\log^* d)$ samples. This separates the sample complexity of approximate-private proper-learners from that of non-private learners.

Tables~\ref{table:upper} and~\ref{table:lower} summarize the currently known bounds on the sample complexity of private learners. Table~\ref{table:upper} specifies {\em general} upper bounds, and table~\ref{table:lower} specifies known upper and lower bounds on the sample complexity of privately learning thresholds over $\{0,1\}^d$. 

\begin{table}[ht]
\centering % used for centering table
\begin{tabular}{c|c|c} % centered columns (3 columns)
\hline\hline
 & Pure-privacy & Approximate-privacy \\ % inserts table 
%heading
\hline % inserts single horizontal line
\begin{tabular}{@{}l@{}}
Proper\\learning
\end{tabular} & $O(\log|C|)$  & $O(\log|C|)$ \\
\hline
\begin{tabular}{@{}l@{}}
Improper\\learning
\end{tabular} & $\Theta(\RepDim(C))$ & $O(\RepDim(C))$  \\ % [1ex] adds vertical space
\hline\hline %inserts single line
\end{tabular}
 % is used to refer this table in the text
\vspace{-1.5ex}
\caption{\small General upper bounds on the sample complexity of private learners for a class $C$.}
\label{table:upper}
\end{table}

\begin{table}[ht]
\centering % used for centering table
\begin{tabular}{c|c|c} % centered columns (3 columns)
\hline\hline
 & Pure-privacy & Approximate-privacy \\ % inserts table 
%heading
\hline % inserts single horizontal line
\begin{tabular}{@{}l@{}}
Proper\\learning
\end{tabular} & $\Theta(d)$  & 
\begin{tabular}{@{}l@{}}
Upper bound: $2^{O(\log^*d)}$\\Lower bound: $\Omega(\log^*d)$
\end{tabular}\\
\hline
\begin{tabular}{@{}l@{}}
Improper\\learning
\end{tabular} & $\Theta(d)$ & 
\begin{tabular}{@{}l@{}}
Upper bound: $2^{O(\log^*d)}$\\Lower bound: $\Omega(1)$
\end{tabular}\\
\hline\hline %inserts single line
\end{tabular}
\vspace{-1.5ex}
\caption{\small Bounds on the sample complexity of private learners for a thresholds over $\{0,1\}^d$. While the VC dimension of this class is constant, its representation dimension is $\Theta(d)$.} % title of Table
\label{table:lower} % is used to refer this table in the text
\end{table}

\subsection{This Work}
In this work we examine an alternative approach for reducing the costs of private learning, inspired by the (non-private) models of {\em semi-supervised learning}~\cite{SemiSupervised} and {\em active learning}~\cite{ActiveLearning}.\footnote{A semi-supervised learner uses a small batch of labeled examples and a large batch of unlabeled examples, whereas an active-learner operates on a large batch of unlabeled example and chooses (maybe adaptively) which examples should be labeled.} In both models, the focus is on reducing the sample complexity of {\em labeled} examples whereas it is assumed that {\em unlabeled} examples can be obtained with a significantly lower cost. 
In this vein, a recent work by Balcan and Feldman~\cite{BF13} suggested a generic conversion of active learners in the model of statistical queries~\cite{Kearns98} into learners that also provide differential privacy. For example, Balcan and Feldman showed an active pure-private proper-learner for the class of thresholds over $\{0,1\}^d$ that uses $O(1)$ labeled examples and $O(d)$ unlabeled examples.

We show that while the unlabeled sample complexity of private learners is subject to the lower bounds mentioned in tables~\ref{table:upper} and~\ref{table:lower}, the {\em labeled} sample complexity is characterized by the VC dimension of the target concept class. We present two generic constructions of private semi-supervised learners via an approach that deviates from most of the research in semi-supervised and active learning: 
(1)
%\begin{itemize}
%\item 
Semi-supervised learning algorithms and heuristics often rely on strong assumptions about the data, e.g., that close points are likely to be labeled similarly, that the data is clustered, or that the data lies on a low dimensional subspace of the input space. In contrast, we work in the standard PAC learning model, and need not make any further assumptions.
%\item 
(2)
Active learners examine their pool of unlabeled data and then choose (maybe adaptively) which  data examples to label. Our learners have no control over which of the sample elements are labeled.
%\end{itemize}

Our main result is that the labeled sample complexity of such learners is characterized by the VC dimension.
Our first generic construction is of learners where the labeled sample complexity is proportional to the VC dimension of the concept class. However, the unlabeled sample complexity of the algorithm is as big as the representation length of domain elements. The learner for a class $C$ starts with an unlabeled database and uses private sanitization to create a synthetic database, with roughly $\VC(C)$ points, that can answer queries in a class related to $C$. It then uses this database to choose a subset of the hypotheses of size $2^{O(\VC(C))}$ and then uses the exponential mechanism~\cite{MT07} to choose from these hypotheses using $O(\VC(C))$ labeled examples.  

As an example, applying this technique with the private sanitizer for threshold functions from~\cite{BNS13b} we get a (semi-supervised) approximate-private proper-learner for thresholds over $\{0,1\}^d$ with optimal $O(1)$ labeled sample complexity and near optimal $2^{O(\log^*d)}$ unlabeled sample complexity. This matches the labeled sample complexity of Balcan and Feldman~\cite{BF13} (ignoring the dependency in all parameters except for $d$), and improves on the unlabeled sample complexity.\footnote{We remark that -- unlike this work -- the focus in~\cite{BF13} is on the dependency of the labeled sample complexity in the approximation parameter. As our learners are non-active, their labeled sample complexity is lower bounded by $\Omega(\frac{1}{\alpha})$ (where $\alpha$ is the approximation parameter).}

Our second construction presents a new technique for decreasing the labeled sample complexity of a given private learner $\AAA$. At the heart of this construction is a technique for choosing (non-privately) a hypothesis using a small labeled database; this hypothesis is used to label a bigger database, which is given to the private learner $\AAA$.

Consider, for example, the concept class $\rectangle_d^{\ell}$ containing all axis-aligned rectangles over $\ell$ dimensions, where each dimension consists of $2^d$ points.
Applying our techniques on the learner from~\cite{BNS13b} results in a non-active semi-supervised private learner with optimal $O(\ell)$ labeled sample complexity and with $\widetilde{O}(\ell^3 \cdot 8^{\log^*d})$ unlabeled sample complexity. This matches the labeled sample complexity of Balcan and Feldman~\cite{BF13}, and improves the unlabeled sample complexity whenever the dimension $\ell$ is not too big (roughly, $\ell \leq \sqrt{d}$).

\paragraph{Private Active Learners.}
We study the labeled sample complexity of private {\em active} learners, i.e., learners that operate on a pool of unlabeled examples (individuals' data) and adaptively query the labels of specific examples. 
As those queries depend on individuals' data, they may breach privacy if exposed.
We, therefore, introduce a stronger definition for private active learners that remedies this potential risk, and show that (most of) our learners satisfy this stronger definition, while the learners of~\cite{BF13} do not. This strong definition has its downside, as we show that (at least in some cases) it introduces a $\frac{1}{\alpha}$ blowup to the labeled sample complexity (where $\alpha$ is the approximation parameter).
On the other hand, when considering private active learners that only satisfy the definition of~\cite{BF13} (which is still a reasonable definition), we show that the labeled sample complexity has no dependency on the privacy parameters.

%\paragraph{Related Work.}
\subsection{Related Work}
Differential privacy was defined in~\cite{DMNS06} and the relaxation to approximate differential privacy is from~\cite{DKMMN06}. Most related to our work is the work on private learning and its sample complexity~\cite{BDMN05,KLNRS08,CH11,DRV10,BBKN12,BNS13,BNS13b,FX14} and the early work on sanitization~\cite{BLR08full}.
Blum et al.~\cite{BDMN05} showed 
that computationally efficient private-learners exist for all concept classes that can be efficiently learned in the {\em statistical queries} model of~\cite{Kearns98}. Kasiviswanathan et al.~\cite{KLNRS08} showed an example of a concept class -- the class of parity functions -- that is not learnable in the statistical queries model but can be learned privately and efficiently. These positive results show that many ``natural''  learning tasks that are efficiently learned non-privately can be learned privately and efficiently. 

\remove{
Chaudhuri and Hsu~\cite{CH11} studied the sample complexity needed for private learning when the data is drawn from a continuous domain. They showed that under these settings there exists a simple
concept class for which any proper learner that uses a finite number of
examples and guarantees pure-privacy fails to satisfy accuracy
guarantee for at least one data distribution.
} %remove

Chaudhuri and Hsu~\cite{CH11} presented upper and lower bounds on the sample complexity of {\em label-private} learners, a relaxation of private
learning where the learner is required to only protect the privacy of the labels in the sample. Following that, Beimel et al.~\cite{BNS13b} showed that the VC dimension completely characterizes the sample complexity of such learners.

\remove{
A characterization for the sample complexity of {\em pure-private} learners was recently given in~\cite{BNS13}, in terms of a new dimension -- the {\em Representation Dimension}, that is, given a class $C$, the number of samples needed and sufficient for privately learning $C$ is $\Theta(\RepDim(C))$. Following~\cite{BNS13}, Feldman and Xiao~\cite{FX14} showed an equivalence between the representation dimension of a concept $C$ and the randomized one-way communication complexity of the evaluation problem for concepts from $C$. Using this equivalence they separated the sample complexity of pure-private learners from that of non-private ones.
}%remove

Dwork et al.~\cite{DRV10} showed how to boost the accuracy of private learning algorithms. That is, given a {\em private} learning algorithm that
has a big classification error, they produced a {\em private} learning algorithm with small error.
Other tools for private learning include, e.g., private SVM~\cite{RBHT09}, private logistic regression~\cite{CM08}, and private empirical risk minimization~\cite{CMS11}.

\section{Preliminaries}
In this section we define differential privacy and semi-supervised (private) learning. Additional preliminaries on the VC dimension and on data sanitization are deferred to the appendix. 

\paragraph{Notation.} We use $O_{\gamma}(g(n))$ as a shorthand for $O(h(\gamma) \cdot g(n))$ for some non-negative function $h$. In informal discussions, we sometimes write $\widetilde{O}(g(n))$ to indicate that $g(n)$ is missing lower order terms. We use $X$ to denote an arbitrary domain, and $X_d$ for the domain $\{0,1\}^d$.

\paragraph{Differential Privacy.} 
Consider a database where each entry contains information pertaining to an individual. An algorithm operating on such databases is said to preserve {\em differential privacy} if its outcome is insensitive to any modification in a single entry. Formally:

%\vspace{-0.05in}
\begin{defn}[Differential Privacy~\cite{DMNS06,DKMMN06}] \label{def:dp} 
Databases $S_1\in X^n$ and $S_2\in X^n$ over a domain $X$ are called {\em neighboring} if they differ in exactly one entry.
A randomized algorithm $\AAA$ is $(\epsilon,\delta)$-differentially private if for all neighboring databases $\db_1,\db_2\in X^n$, and for all sets $F$ of outputs,
\begin{eqnarray}
\label{eqn:diffPrivDef}
  & \Pr[\AAA(\db_1) \in F] \leq \exp(\epsilon) \cdot \Pr[\AAA(\db_2) \in F] + \delta.  &
\end{eqnarray}
The probability is taken over the random coins of $\AAA$. 
When $\delta{=}0$ we omit it and say that $\AAA$ preserves {\em pure} differential privacy, otherwise (when $\delta>0$) we say that $\AAA$ preserves {\em approximate} differential privacy. 
\end{defn}

See Appendix~\ref{sec:dp_mech} for basic differentially private mechanisms.

\paragraph{Semi-Supervised PAC Learning.}

The standard PAC model (and similarly private PAC) focuses on learning a class of concepts from a sample of labeled examples. In a situation where labeled examples are significantly more costly than unlabeled ones, it is natural to attempt to use a combination of labeled and unlabeled data to reduce the number of labeled examples needed.  Such learners may have no control over which of the examples are labeled, as in {\em semi-supervised learning}, or may specifically choose which examples to label, as in {\em active learning}. In this section we focus on semi-supervised learning. Active learning will be discussed in Section~\ref{sec:privActive}.

A concept $c:X\rightarrow \{0,1\}$ is a predicate that labels {\em examples} taken from the domain $X$ by either 0 or 1.  A \emph{concept class} $C$ over $X$ is a set of concepts (predicates) mapping $X$ to $\{0,1\}$. A semi-supervised learner is given $n$ examples sampled according to an unknown probability distribution $\mu$ over $X$, where $m\leq n$ of these examples are labeled according to an unknown {\em target} concept $c\in C$. The learner succeeds if it outputs a hypothesis $h$ that is a good approximation of the target concept according to the distribution $\mu$. Formally:

\begin{defn}
Let $c$ and $\mu$ be a concept and a distribution over a domain $X$.
The {\em generalization error} of a hypothesis $h:X\rightarrow\{0,1\}$ w.r.t.\ $c$ and $\mu$ is defined as $\error_{\mu}(c,h)=\Pr_{x \sim \mu}[h(x)\neq c(x)].$
When $\error_{\mu}(c,h)\leq\alpha$ we say that $h$ is {\em $\alpha$-good} for $c$ and $\mu$.
\end{defn}

\begin{defn}[Semi-Supervised~\cite{Valiant84, SemiSupervised}]\label{def:PAC}
Let $C$ be a concept class over a domain $X$, and let $\AAA$ be an algorithm operating on (partially) labeled databases.
Algorithm $\AAA$ is an {\em $(\alpha,\beta,n,m)$-SSL (semi-supervised learner)} for $C$ if for all concepts $c \in C$ and all distributions $\mu$ on $X$ the following holds.

Let $D=(x_i,y_i)_{i=1}^n\in(X\times\{0,1,\bot\})^n$ be a database s.t.\ (1)~each $x_i$ is drawn i.i.d.\ from $\mu$; (2)~in the first $m$ entries $y_i=c(x_i)$; (3)~in the last $(n-m)$ entries $y_i=\bot$. Then,
$$\Pr[\AAA(D){=}h \text{ s.t.\ } \error_{\mu}(c,h) > \alpha] \leq \beta.$$
The probability is taken over the choice of
the samples from $\mu$ and the coin tosses of $\AAA$.
\end{defn}

If a semi-supervised learner is restricted to only output hypotheses from the target concept class $C$, then it is called a {\em proper} learner. Otherwise, it is called an {\em improper} learner.
We sometimes refer to the input of a semi-supervised learner as two databases $D\in (X \times \{\bot\})^{n-m}$ and $S\in(X \times \{0,1\})^m$, where $m$ and $n$ are the {\em labeled} and {\em unlabeled} sample complexities of the learner.

%\vspace{-0.05in}
\begin{defn}
Given a {\em labeled} sample $S=(x_i,y_i)_{i=1}^m$, the {\em empirical error} of a hypothesis $h$ on $S$ is
$\error_S(h) = \frac{1}{m} |\{i : h(x_i) \neq y_i\}|$.
Given an {\em unlabeled} sample $D=(x_i)_{i=1}^n$ and a target concept $c$, the {\em empirical error} of $h$ w.r.t.\ $D$ and $c$ is
$\error_D(h,c) = \frac{1}{n} |\{i : h(x_i) \neq c(x_i)\}|$.
\end{defn}

Semi-supervised learning algorithms operate on a (partially) labeled sample with the goal of choosing a hypothesis with a small {\em generalization} error. Standard arguments in learning theory (see Appendix~\ref{sec:VC}) state that the generalization of a hypothesis $h$ and its {\em empirical} error (observed on a large enough sample) are similar. Hence, in order to output a hypothesis with small generalization error it suffices to output a hypothesis with small empirical error.

\paragraph{Agnostic Learner.}
Consider an SSL for an {\em unknown} class $C$ that uses a (known) hypotheses class $H$. If $H\neq C$, then a hypothesis with small empirical error might not exist in $H$. Such learners are referred to in the literature as {\em agnostic}-learners, and are only required to produce a hypothesis $f\in H$ (approximately) minimizing $\error_{\mu}(c,f)$, where $c$ is the (unknown) target concept.

\begin{defn}[Agnostic Semi-Supervised]\label{def:PACagnostic}
Let $H$ be a concept class over a domain $X$, and let $\AAA$ be an algorithm operating on (partially) labeled databases.
Algorithm $\AAA$ is an {\em $(\alpha,\beta,n,m)$-agnostic-SSL} using $H$ if for all concepts $c$ (not necessarily in $H$) and all distributions $\mu$ on $X$ the following holds.

Let $D=(x_i,y_i)_{i=1}^n\in(X\times\{0,1,\bot\})^n$ be a database s.t.\ (1)~each $x_i$ is drawn i.i.d.\ from $\mu$; (2)~in the first $m$ entries $y_i=c(x_i)$; (3)~in the last $(n-m)$ entries $y_i=\bot$. Then, $\AAA(D)$ outputs a hypothesis $h\in H$ satisfying
$\Pr[\error_{\mu}(c,h)  \leq \min_{f\in H}\{\error_{\mu}(c,f)\}+\alpha] \geq 1-\beta.$
The probability is taken over the choice of
the samples from $\mu$ and the coin tosses of $\AAA$.
\end{defn}

\paragraph{Private Semi-Supervised PAC learning.}\label{sec:PPAC}

Similarly to~\cite{KLNRS08} we define private semi-supervised learning as the combination of Definitions~\ref{def:dp} and~\ref{def:PAC}.
%\vspace{-0.05in}
\begin{defn}[Private Semi-Supervised]
\label{def:private-SSL}
Let $\AAA$ be an algorithm that gets an input $S\in(X\times\{0,1,\bot\})^n$. Algorithm $\AAA$ is an {\em $(\alpha,\beta,\epsilon,\delta,n,m)$-PSSL (private SSL)} for a concept class $C$ over $X$ if $\AAA$ is an $(\alpha,\beta,n,m)$-SSL for $C$ {\em and} $\AAA$ is $(\epsilon,\delta)$-differentially private.
\end{defn}

\paragraph{Active Learning.} Semi-supervised learners are a subset of the larger family of {\em active learners}. Such learners can adaptively request to reveal the labels of specific examples. See formal definition and discussion in Section~\ref{sec:privActive}.

\section{A Generic Construction Achieving Low Labeled Sample Complexity}
\label{sec:semiSuper}

We next study the labeled sample complexity of private semi-supervised learners. We begin with a generic algorithm showing that for every concept class $C$ there exist a pure-private proper-learner with labeled sample complexity (roughly) $\VC(C)$.
%That is, we show that the VC dimension characterizes the labeled sample complexity of private learners.
This algorithm, called $GenericLearner$, is described in Algorithm~\ref{alg:genericPrivate}.
The algorithm operates on a labeled database $S$ and on an unlabeled database $D$. First, the algorithm produces a sanitization $\widetilde{D}$ of the unlabeled database $D$ w.r.t.\ $C^{\oplus}$ (to be defined). Afterwards, the algorithm uses $\widetilde{D}$ to construct a small set of hypotheses $H$ (we will show that $H$ contains at least one good hypothesis). Finally, the algorithm uses the exponential mechanism to choose a hypothesis out of $H$.

Similar ideas have appeared in~\cite{CH11,BNS13b} in the context of {\em label-private} learners, i.e., learners that are only required to protect the privacy of the {\em labels} in the sample (and not the privacy of the elements themselves). Like $GenericLearner$, the learners of~\cite{CH11,BNS13b} construct a small set of hypotheses $H$ that ``covers'' the hypothesis space and then use the exponential mechanism in order to choose a hypothesis $h\in H$. However, $GenericLearner$ differs in that it protects the privacy of the entire sample (both the labels and the elements themselves). 

%We begin with the following simple notation.

%\vspace{-0.05in}
\begin{defn}
Given two concepts $h,f\in C$, we denote $(h {\oplus} f): X_d \rightarrow \{0,1\} $, where $(h {\oplus} f)(x)=1$ if and only if $h(x)\neq f(x)$. Let $C^{\oplus}=\{ (h {\oplus} f) \; : \; h,f\in C \}.$
\end{defn}

To preserve the privacy of the examples in $D$, we first create a sanitized version of it -- $\widetilde{D}$.
If the entries of $D$ are drawn i.i.d.\ according to the underlying distribution (and if $D$ is big enough), 
then a hypothesis with small empirical error on $D$ also has small generalization error (see Theorem~\ref{thm:generalization}). Our learner classifies the sanitized database $\widetilde{D}$ with small error, thus we require that a small error on $\widetilde{D}$ implies a small error on $D$. Specifically, if $c$ is the target concept, then we require that for every $f \in C$, 
$\error_D(f,c) = \frac{1}{|D|} \left| \{ x\in D \; : \; f(x)\neq c(x)   \} \right|$ 
is approximately the same as
$\error_{\widetilde{D}}(f,c) = \frac{1}{|\widetilde{D}|} \left| \{ x\in \widetilde{D} \; : \; f(x)\neq c(x)   \} \right|$.
Observe that this is exactly what we would get from a sanitization of $D$ w.r.t.\ the concept class $C^{\oplus c}=\{ (f {\oplus} c) \; : \; f\in C \}$.
As the target concept $c$ is unknown, we let $\widetilde{D}$ be a sanitization of $D$ w.r.t.\ $C^{\oplus}$, which contains $C^{\oplus c}$.

To apply the sanitization of Blum et al.~\cite{BLR08full} to $D$ w.r.t.\ the class $C^{\oplus}$, we analyze the VC dimension of $C^{\oplus}$ in the next observation.

\begin{obs}\label{obs:vcdim}
For any concept class $C$ over $X_d$ it holds that $\VC(C^{\oplus})=O(\VC(C))$.
\end{obs}

\begin{proof}
Recall that the projection of $C$ on a set of domain points $B=\{b_1,\ldots,b_\ell\}\subseteq X_d$ is 
$\Pi_C(B)=\{\left\langle c(b_1),\ldots,c(b_\ell) \right\rangle :c\in C\}.$
Now note that for every $B=\{b_1,\ldots,b_\ell\}\subseteq X_d$
\begin{eqnarray*}
\Pi_{C^{\oplus}}(B)&=&\{ \left\langle (h \oplus f)(b_1),\ldots,(h \oplus f)(b_\ell) \right\rangle :h,f\in C\}\\
&=&\{ \left\langle h(b_1),\ldots,h(b_\ell)  \right\rangle  \oplus  \left\langle  f(b_1),\ldots,f(b_\ell)  \right\rangle  :h,f\in C\}\\
&=&\{ \left\langle  h(b_1),...,h(b_\ell)  \right\rangle  :h{\in} C\} \oplus \{  \left\langle  f(b_1),...,f(b_\ell)  \right\rangle  :f{\in} C\}\\
&=&\Pi_{C}(B) \oplus \Pi_{C}(B).
\end{eqnarray*}

Therefore, by Sauer's lemma~\ref{thm:sauer}, $|\Pi_{C^{\oplus}}(B)| \leq |\Pi_{C}(B)|^2 \leq \left(\frac{e \ell}{\VC(C)}\right)^{2\VC(C)}$. 
Hence, for $C^{\oplus}$ to shatter a subset $B\subseteq X_d$ of size $\ell$ it must be that $\left(\frac{e \ell}{\VC(C)}\right)^{2\VC(C)} \geq 2^\ell$. For $\ell\geq 10\VC(C)$ this inequality does not hold, and we can conclude that $\VC(C^{\oplus})\leq10\VC(C)$.
\end{proof}

\begin{algorithm}
\caption{$GenericLearner$}\label{alg:genericPrivate}
{\bf Input:} parameter $\epsilon$, an unlabeled database $D=(x_i)_{i=1}^{n-m}$, and a labeled database $S=(x_i,y_i)_{i=1}^m$.
\begin{enumerate}[rightmargin=10pt,itemsep=1pt]

\item Initialize $H=\emptyset$.

\item Construct an $\epsilon$-private sanitization $\widetilde{D}$ of $D$ w.r.t.\ $C^{\oplus}$, where $|\widetilde{D}|=O\left( \frac{\VC(C^{\oplus})}{\alpha^2}\log(\frac{1}{\alpha}) \right) = O\left( \frac{\VC(C)}{\alpha^2}\log(\frac{1}{\alpha}) \right)$ (e.g., using Theorem~\ref{thm:BlumUp}).

\item Let $B=\{b_1,\ldots,b_\ell\}$ be the set of all (unlabeled) points appearing at least once in $\widetilde{D}$.

\item For every $(z_1,\ldots,z_\ell)\in \Pi_C(B)=\{\left( c(j_1),\ldots,c(j_\ell) \right) :c\in C\}$, add to $H$ an arbitrary concept $c\in C$ s.t.\ $c(b_i)=z_i$ for every $1\leq i\leq\ell$.

\item Choose and return $h\in H$ using the exponential mechanism with inputs $\epsilon,H,S$.
\end{enumerate}
\end{algorithm}

\begin{thm}\label{thm:SampleComplexity}
Let $C$ be a concept class over $X_d$.
For every $\alpha,\beta,\epsilon$, there exists an $(\alpha,\beta,\epsilon,\delta{=}0,n,m)$-private semi-supervised proper-learner for $C$, where 
$m=O\left(\frac{\VC(C)}{\alpha^3\epsilon}\log(\frac{1}{\alpha}) +\frac{1}{\alpha\epsilon}\log(\frac{1}{\beta})\right)$, and
$n=O\left(\frac{d\cdot\VC(C)}{\alpha^3\epsilon}\log(\frac{1}{\alpha}) +\frac{1}{\alpha\epsilon}\log(\frac{1}{\beta})\right)$.
The learner might not be efficient.
\end{thm}

\begin{proof}
Note that $GenericLearner$ only accesses $D$ via a sanitizer, and only accesses $S$ using the exponential mechanism (on Step~5). As each of those two mechanisms is $\epsilon$-differentially private, and as $D$ and $S$ are two disjoint samples, $GenericLearner$ is $\epsilon$-differentially private.
We, thus, only need to prove that with high probability the learner returns a good hypothesis.

Fix a target concept $c\in C$ and a distribution $\mu$ over $X$, and define the following three ``good'' events:
\begin{enumerate}[label=$E_{\arabic*}:$]

\item For every $h\in C$ it holds that $|\error_S(h)-\error_{\widetilde{D}}(h,c)|\leq\frac{3\alpha}{5}$.

\item The exponential mechanism chooses an $h\in H$ such that $\error_S(h) \leq \frac{\alpha}{5} + \min_{f\in H}\left\{\error_S(f)\right\}$.

\item For every $h\in H$ s.t.\ $\error_S(h)\leq\frac{4\alpha}{5}$, it holds that $\error_{\mu}(c,h)\leq\alpha$.
\end{enumerate}

We first observe that when these three events happen algorithm $GenericLearner$ returns an $\alpha$-good hypothesis:
For every $(y_1,\ldots,y_\ell)\in \Pi_C(B)$, algorithm $GenericLearner$ adds to $H$ a hypothesis $f$ s.t.\ $\forall 1\leq i \leq \ell,\;f(b_i)=y_i$.
In particular, $H$ contains a hypothesis $h^*$ s.t.\ $h^*(x)=c(x)$ for every $x\in B$, that is, a hypothesis $h^*$ s.t.\ $\error_{\widetilde{D}}(h^*,c)=0$.
As event $E_1$ has occur we have that this $h^*$ satisfies $\error_S(h^*)\leq \frac{3\alpha}{5}$.
Thus, event $E_1 \cap E_2$ ensures that algorithm $GenericLearner$ chooses (using the exponential mechanism) a hypothesis $h\in H$ s.t.\ $\error_S(h)\leq\frac{4\alpha}{5}$. Event $E_3$ ensures, therefore, that this $h$ satisfies $\error_{\mu}(c,h)\leq\alpha$.
We will now show $E_1  \cap E_2 \cap E_3$ happens with high probability.

Standard arguments in learning theory state that (w.h.p.) the empirical error on a (large enough) random sample is close to the generalization error (see Theorem~\ref{thm:generalization}).
Specifically, by setting $n$ and $m$ to be at least $\frac{1250}{\alpha^2}\VC(C)\ln(\frac{25}{\alpha\beta})$, Theorem~\ref{thm:generalization} ensures that
with probability at least $(1-\frac{2}{5}\beta)$, for every $h\in C$ the following two inequalities hold.
\begin{eqnarray}
&&|\error_S(h)-\error_{\mu}(h,c)|\leq\frac{\alpha}{5} \label{eq:errorSD}\\
&&|\error_D(h,c)-\error_{\mu}(h,c)|\leq\frac{\alpha}{5} \label{eq:errorDD}
\end{eqnarray}
Note that Event $E_3$ occurs whenever Inequality~(\ref{eq:errorSD}) holds (since $H\subseteq C$).
Moreover, by setting the size of the unlabeled database $(n-m)$ to be at least
\begin{eqnarray*}
(n-m)&\geq& O\left(\frac{d \cdot \VC(C^{\oplus})\log(\frac{1}{\alpha})}{\alpha^3\epsilon}+\frac{\log(\frac{1}{\beta})}{\epsilon\alpha}\right)\\
&=&O\left(\frac{d \cdot \VC(C)\log(\frac{1}{\alpha})}{\alpha^3\epsilon}+\frac{\log(\frac{1}{\beta})}{\epsilon\alpha}\right).
\end{eqnarray*}
Theorem~\ref{thm:BlumUp} ensures that with probability at least $(1-\frac{\beta}{5})$ for every $(h\oplus f)\in C^{\oplus}$ (i.e., for every $h,f\in C$) it holds that
\begin{eqnarray*}
\frac{\alpha}{5}&\geq& | Q_{(h{\oplus} f)}(D) - Q_{(h{\oplus} f)}(\widetilde{D}) |  \\
&=& \left| \frac{|\{x\in D : (h{\oplus} f)(x){=}1 \}|}{|D|} - \frac{|\{x\in \widetilde{D} : (h{\oplus} f)(x){=}1 \}|}{|\widetilde{D}|}  \right|  \\
&=& \left| \frac{|\{x\in D : h(x){\neq} f(x) \}|}{|D|} - \frac{|\{x\in \widetilde{D} : h(x){\neq} f(x) \}|}{|\widetilde{D}|}  \right|  \\
&=& \left| \error_D(h,f) - \error_{\widetilde{D}}(h,f)  \right|. 
\end{eqnarray*}
In particular, for every $h\in C$ it holds that 
\begin{eqnarray}
\left| \error_D(h,c) - \error_{\widetilde{D}}(h,c)  \right|\leq\frac{\alpha}{5}.
\label{eq:errorDD'}
\end{eqnarray}
Therefore (using Inequalities~(\ref{eq:errorSD}),(\ref{eq:errorDD}),(\ref{eq:errorDD'}) and the triangle inequality), Event $E_1\cap E_3$ occurs with probability at least $(1-\frac{3\beta}{5})$.

The exponential mechanism ensures that the probability of event $E_2$ is at least $1-|H| \cdot \exp(-\epsilon \alpha m /10)$ (see Proposition \ref{prop:expMech}).
Note that $\log|H|\leq|B|\leq|\widetilde{D}| = O\left( \frac{\VC(C)}{\alpha^2}\log(\frac{1}{\alpha}) \right)$. Therefore, for
$m \geq O\left(\frac{\VC(C)}{\alpha^3\epsilon}\log(\frac{1}{\alpha}) +\frac{1}{\alpha\epsilon}\log(\frac{1}{\beta})\right)$,
Event $E_2$ occurs with probability at least $(1-\frac{\beta}{5})$. 

All in all, setting
$n\geq O\left(\frac{d \cdot \VC(C)\log(\frac{1}{\alpha})}{\alpha^3\epsilon}+\frac{\log(\frac{1}{\beta})}{\epsilon\alpha}\right)$,
and
$m \geq O\left(\frac{\VC(C)}{\alpha^3\epsilon}\log(\frac{1}{\alpha}) +\frac{1}{\alpha\epsilon}\log(\frac{1}{\beta})\right)$,
ensures that the probability of $GenericLearner$ failing to output an $\alpha$-good hypothesis is at most $\beta$.
\end{proof}

Note that the labeled sample complexity in Theorem~\ref{thm:SampleComplexity} is optimal (ignoring the dependency in $\alpha,\beta,\epsilon$), as even without the privacy requirement every PAC learner for a class $C$ must have {\em labeled} sample complexity $\Omega(\VC(C))$.
However, the unlabeled sample complexity is as big as the representation length of domain elements, that is, $O(d\cdot\VC(C))$.
Such a blowup in the unlabeled sample complexity is unavoidable in any generic construction of pure-private learners.\footnote{
Feldman and Xiao~\cite{FX14} showed an example of a concept class $C$ over $X_d$ for which every pure-private learner must have unlabeled sample complexity $\Omega(\VC(C)\cdot d)$. Hence, as a function of $d$ and $\VC(C)$, the unlabeled sample complexity in Theorem~\ref{thm:SampleComplexity} is the best possible for a generic construction of pure-private learners.} 

To show the usefulness of Theorem~\ref{thm:SampleComplexity}, we consider the concept class $\thresh_d$ defined as follows. For $0\leq j\leq 2^d$ let $c_j:X_d \rightarrow\{0,1\}$ be defined as $c_j(x)=1$ if $x<j$ and $c_j(x)=0$ otherwise. Define the concept class $\thresh_d = \{c_j \, : \, 0\leq j\leq 2^d\}$.
Balcan and Feldman~\cite{BF13} showed an efficient pure-private proper-learner for $\thresh_d$ with labeled sample complexity $O_{\alpha,\beta,\epsilon}(1)$ and unlabeled sample complexity $O_{\alpha,\beta,\epsilon}(d)$. At the cost of preserving approximate-privacy, and using the efficient approximate-private sanitizer for thresholds from~\cite{BNS13b} (in Step~2 of Algorithm $GenericLearner$ instead on the sanitizer of~\cite{BLR08full}), we get the following lemma (as $GenericLearner$ requires unlabeled examples only in Step~2, and the sanitizer of~\cite{BNS13b} requires a database of size $\widetilde{O}_{\alpha,\beta,\epsilon,\delta}(8^{\log^*d})$).

\begin{cor}\label{intervalLeraner}
There exists an efficient approximate-private proper-learner for $\thresh_d$ with labeled sample complexity $O_{\alpha,\beta,\epsilon}(1)$ and unlabeled sample complexity $\widetilde{O}_{\alpha,\beta,\epsilon,\delta}(8^{\log^*d})$.
\end{cor}

Beimel et al.~\cite{BNS13b} showed an efficient approximate-private proper-learner for $\thresh_d$ with (both labeled and unlabeled) sample complexity $\widetilde{O}_{\alpha,\beta,\epsilon,\delta}(16^{\log^*d})$. The learner from Corollary~\ref{intervalLeraner} has similar unlabeled sample complexity, but improves on the labeled complexity.

%\sectionnegative
\section{Boosting the Labeled Sample Complexity of Private Learners}\label{sec:boost}

We now show a generic transformation of a private learning algorithm $\AAA$ for a class $C$ into a private learner with reduced labeled sample complexity (roughly $\VC(C)$), while maintaining its unlabeled sample complexity. This transformation could be applied to a proper or an improper learner, and to a learner that preserves pure or approximated privacy.

The main ingredient of the transformation is algorithm $LabelBoostProcedure$ (Algorithm~\ref{alg:LabelBoostProcedure}), where the labeled sample complexity is reduced logarithmically. We will later use this procedure iteratively to get our learner with labeled sample complexity $O_{\alpha,\beta,\epsilon}(\VC(C))$.

Given a partially labeled sample $B$ of size $n$, algorithm $LabelBoostProcedure$ chooses a small subset $H$ of $C$ that strongly depends on the points in $B$ so outputting a hypothesis $h\in H$ may breach privacy. Nevertheless, $LabelBoostProcedure$ does choose a good hypothesis $h\in H$ (using the exponential mechanism) and use it to relabel part of the sample $B$.
In Lemma~\ref{lemma:TransformationPrivacy}, we analyze the privacy guarantees of Algorithm $LabelBoostProcedure$.
\remove{We do not know if Algorithm $LabelBoostProcedure$ is a learner. To achieve a learner we add sampling stages to the algorithm. This is done in Algorithm $LabelBoost$, which also applies Algorithm $LabelBoostProcedure$ iteratively in order to further reduce the labeled sample complexity.}

\begin{algorithm}
\caption{$LabelBoostProcedure$}\label{alg:LabelBoostProcedure} 
{\bf Input:} A partially labeled database $B=S{\circ}T{\circ}D\in(X\times\{0,1,\bot\})^*$.
\begin{enumerate}[label=\gray{\%},topsep=-10pt,rightmargin=10pt,leftmargin=12pt]
\item \gray{
We assume that the first portion of $B$ (denoted as $S$) contains labeled examples. Our goal is to output a similar database where both $S$ and $T$ are labeled. 
}
\end{enumerate}
\begin{enumerate}[rightmargin=10pt,itemsep=1pt,topsep=0pt]

\item Initialize $H=\emptyset$.

\item Let $P=\{p_1,\ldots,p_\ell\}$ be the set of all points $p\in X$ appearing at least once in $S{\circ}T$.

\item For every $(z_1,\ldots,z_\ell)\in \Pi_C(P)=\{\left( c(p_1),\ldots,c(p_\ell) \right) :c\in C\}$, add to $H$ an arbitrary concept $c\in C$ s.t.\ $c(p_i)=z_i$ for every $1\leq i\leq\ell$.

\item \label{step:Oneexpmech} Choose $h\in H$ using the exponential mechanism with privacy parameter $\epsilon{=}1$, solution set $H$, and the database $S$.

\item \label{step:Onerelabel} Relabel $S{\circ}T$ using $h$, and denote this relabeled database as $(S{\circ}T)^h$, that is, if $S{\circ}T=(x_i,y_i)_{i=1}^t$ then $(S{\circ}T)^h=(x_i,y'_i)_{i=1}^t$ where $y'_i=h(x_i)$.

\item \label{step:OneAAA} Output $(S{\circ}T)^h {\circ} D$.

\end{enumerate}
\end{algorithm}

\begin{lem}\label{lemma:TransformationPrivacy}
Let $\AAA$ be an $(\epsilon,\delta)$-differentially private algorithm operating on partially labeled databases.
Construct an algorithm $\BBB$ that on input a database $S{\circ}T{\circ}D\in(X\times\{0,1,\bot\})^*$ applies $\AAA$ on the outcome of $LabelBoostProcedure(S{\circ}T{\circ}D)$.
Then, $\BBB$ is $(\epsilon+3,4e\delta)$-differentially private.
\end{lem}

\begin{proof}
Consider the executions of $\BBB$ on two neighboring inputs $S_1{\circ}T_1{\circ}D_1$ and $S_2{\circ}T_2{\circ}D_2$. If these two neighboring inputs differ (only) on the last portion $D$ then the executions of $LabelBoostProcedure$ on these neighboring inputs are identical, and hence Inequality~(\ref{eqn:diffPrivDef}) (approximate differential privacy) follows from the privacy of $\AAA$. We, therefore, assume that $D_1=D_2=D$ (and that $S_1{\circ}T_1,S_2{\circ}T_2$ differ in at most one entry).

Denote by $H_1,P_1$ and by $H_2,P_2$ the elements $H,P$ as they are in the executions of algorithm $LabelBoostProcedure$ on $S_1{\circ}T_1{\circ}D$ and on $S_2{\circ}T_2{\circ}D$.
The main difficulty in proving differential privacy is that $H_1$ and $H_2$ can significantly differ. We show, however, that the distribution on relabeled databases $(S{\circ}T)^h$ generated in Step~\ref{step:Onerelabel} of the two executions are similar in the sense that for each relabeled database in one of the distributions there exist one or two databases in the other s.t.\ (1) all these databases have, roughly, the same probability, and (2) they differ on at most one entry. Thus, executing the differentially private algorithm $\AAA$ on $(S{\circ}T)^h {\circ} D$ preserves differential privacy. We now make this argument formal.

Note that $|P_1\setminus P_2|\in\{0,1\}$, and let $\ppp_1$ be the element in $P_1\setminus P_2$ if such an element exists. If this is the case, then $\ppp_1$ appears exactly once in $S_1{\circ}T_1$.
Similarly, let $\ppp_2$ be the element in $P_2 \setminus P_1$ if such an element exists. Let $K=P_1\cap P_2$, hence $P_i = K$ or $P_i = K \cup \{\ppp_i\}$.  Therefore, $|\Pi_C(K)|\leq|\Pi_C(P_i)|\leq2|\Pi_C(K)|$. Thus, $|H_1|\leq2|H_2|$ and similarly $|H_2|\leq2|H_1|$.

More specifically, for every $\vec{z}\in\Pi_C(K)$ there are either one or two (but not more) hypotheses in $H_1$ that agree with $\vec{z}$ on $K$.
We denote these one or two hypotheses by $h_{1,\vec{z}}$ and $h'_{1,\vec{z}}$, which may be identical if only one unique hypothesis exists.
Similarly, we denote $h_{2,\vec{z}}$ and $h'_{2,\vec{z}}$ as the hypotheses corresponding to $H_2$.
For every $\vec{z}\in\Pi_C(K)$ we have that $|q(S_i,h_{i,\vec{z}})-q(S_i,h'_{i,\vec{z}})|\leq1$ because if $h_{i,\vec{z}}=h'_{i,\vec{z}}$ then the difference is clearly zero and otherwise they differ only on $\ppp_i$, which appears at most once in $S_i$.
Moreover, for every $\vec{z}\in\Pi_C(K)$ we have that $|q(S_1,h_{1,\vec{z}})-q(S_2,h_{2,\vec{z}})|\leq1$ because $h_{1,\vec{z}}$ and $h_{2,\vec{z}}$ disagree on at most two points $\ppp_1, \ppp_2$ such that at most one of them appears in $S_1$ and at most one of them appears in $S_2$. The same is true for every pair in $\{ h_{1,\vec{z}}, h'_{1,\vec{z}} \} \times \{ h_{2,\vec{z}} , h'_{2,\vec{z}} \}$.

Let $w_{i,\vec{z}}$ be the probability that the exponential mechanism chooses $h_{i,\vec{z}}$ or $h'_{i,\vec{z}}$ in Step~\ref{step:Oneexpmech} of the execution on $S_i{\circ}T_i{\circ}D$. We get that for every $\vec{z}\in\Pi_C(K)$,
\begin{eqnarray*}
w_{1,\vec{z}}
%\Pr\left[  \begin{array}{c}
%	\text{The exponential}\\
%	\text{mechanism chooses}\\
%	\text{$h_{1,\vec{z}}$ or $h'_{1,\vec{z}}$}\\
%	\text{on Step~\ref{step:expmech} of the}\\
%	\text{execution on $(N,D_1)$}
%\end{array} \right]
&\leq& \frac{\exp(\frac{1}{2}\cdot q(S_1,h_{1,\vec{z}}))+\exp(\frac{1}{2}\cdot q(S_1,h'_{1,\vec{z}}))}{\sum_{f\in H_1}{\exp(\frac{1}{2}\cdot q(S_1,f))}}\\
&\leq& \frac{\exp(\frac{1}{2}\cdot q(S_1,h_{1,\vec{z}}))+\exp(\frac{1}{2}\cdot q(S_1,h'_{1,\vec{z}}))}{\sum_{\vec{r}\in\Pi_C(K)}{\exp(\frac{1}{2}\cdot q(S_1,h_{1,\vec{r}}))}}\\
&\leq& \frac{\exp(\frac{1}{2}\cdot [q(S_2,h_{2,\vec{z}})+1])+\exp(\frac{1}{2}\cdot [q(S_2,h'_{2,\vec{z}})+1])}{\frac{1}{2}\sum\limits_{\vec{r}\in\Pi_C(K)}\left({\exp(\frac{q(S_2,h_{2,\vec{r}})-1}{2})+\exp(\frac{q(S_2,h'_{2,\vec{r}})-1}{2})}\right)}\\
&\leq& 2 e\cdot \frac{\exp(\frac{1}{2}\cdot [q(S_2,h_{2,\vec{z}})])+\exp(\frac{1}{2}\cdot [q(S_2,h'_{2,\vec{z}})])}{\sum_{f\in H_2}{\exp(\frac{1}{2}\cdot q(S_2,f))}}\\
&\leq& 4e \cdot
%\\
%&\leq& 4e \cdot 
w_{2,\vec{z}}.
%\Pr\left[  \begin{array}{c}
%	\text{The exponential}\\
%	\text{mechanism chooses}\\
%	\text{$h_{2,\vec{z}}$ or $h'_{2,\vec{z}}$}\\
%	\text{on Step~4 of the}\\
%	\text{execution on $(N,D_2)$}
%\end{array} \right]
\end{eqnarray*}

%Next note the three possible cases for $S_1$ and $S_2$. They can be identical databases, or they can be neighboring databases (of the same size), or one can be created from the other by adding/removing one element.

We can now conclude the proof by noting that for every $\vec{z}\in\Pi_C(K)$ the databases $(S_1{\circ}T_1)^{h_{1,\vec{z}}}$ and $(S_2{\circ}T_2)^{h_{2,\vec{z}}}$ are neighboring, and, therefore, $(S_1{\circ}T_1)^{h_{1,\vec{z}}} {\circ}D$ and $(S_2{\circ}T_2)^{h_{2,\vec{z}}} {\circ}D$ are neighboring.
For every $\vec{z}\in\Pi_C(K)$, let $\hhh_{i,\vec{z}}$ denote the event that the exponential mechanism chooses $h_{i,\vec{z}}$ or $h'_{i,\vec{z}}$ in Step~\ref{step:Oneexpmech} of the execution on $S_i{\circ}T_i{\circ}D$.
By the privacy properties of algorithm $\AAA$ we have that for any set $F$ of possible outputs of algorithm $\BBB$ 
\begin{eqnarray*}
\Pr[\BBB\left( S_1{\circ}T_1{\circ}D \right)\in F]&=& \sum_{\vec{z}\in\Pi_C(K)} w_{1,\vec{z}} 
%\Pr\left[  \begin{array}{c}
%	\text{The exponential}\\
%	\text{mechanism chooses}\\
%	\text{$h_{1,\vec{z}}$ or $h'_{1,\vec{z}}$}\\
%	\text{on Step~\ref{step:expmech} of the}\\
%	\text{execution on $(N,D_1)$}
%\end{array} \right]
\cdot\Pr\left[\AAA\left( (S_1{\circ}T_1)^h {\circ}D \right) \in F \Big| \hhh_{1,\vec{z}} \right]\\
&\leq& \sum_{\vec{z}\in\Pi_C(K)}4e \; w_{2,\vec{z}}
%\Pr\left[  \begin{array}{c}
%	\text{The exponential}\\
%	\text{mechanism chooses}\\
%	\text{$h_{2,\vec{z}}$ or $h'_{2,\vec{z}}$}\\
%	\text{on Step~4 of the}\\
%	\text{execution on $(N,D_2)$}
%\end{array} \right]
\left(e^{\epsilon}\Pr\left[\AAA\left(  (S_2{\circ}T_2)^h {\circ}D \right) \in F \Big| \hhh_{2,\vec{z}} \right]+\delta\right)\\
&\leq&  e^{\epsilon+3}\cdot\Pr[\BBB\left( S_2{\circ}T_2{\circ}D \right)\in F]+4e\delta. 
\end{eqnarray*}
\end{proof}

Consider an execution of $LabelBoostProcedure$ on a database $S{\circ}T{\circ}D$, and assume that the examples in $S$ are labeled by some target concept $c\in C$.
Recall that for every possible labeling $\vec{z}$ of the elements in $S$ and in $T$, algorithm $LabelBoostProcedure$ adds to $H$ a hypothesis from $C$ that agrees with $\vec{z}$.
In particular, $H$ contains a hypothesis that agrees with the target concept $c$ on $S$ (and on $T$). That is, $\exists f\in H$ s.t.\ $\error_S(f)=0$.
Hence, the exponential mechanism (on Step~\ref{step:Oneexpmech}) chooses (w.h.p.) a hypothesis $h\in H$ s.t.\ $\error_S(h)$ is small, provided that $|S|$ is roughly $\log|H|$, which is roughly $\VC(C)\cdot\log(|S|+|T|)$ by Sauer's lemma.
So, algorithm $LabelBoostProcedure$ takes an input database where only a small portion of it is labeled, and returns a similar database in which the labeled portion grows exponentially.

\begin{claim}\label{claim:LabelBoostProcedureUtility}
Fix $\alpha$ and $\beta$, and let $S{\circ}T{\circ}D$ be s.t.\ $S$ is labeled by some target concept $c\in C$, and s.t.\ 
$$|T|\leq\frac{\beta}{e} \VC(C)\exp(\frac{\alpha |S|}{2\VC(C)})-|S|.$$
Consider the execution of $LabelBoostProcedure$ on $S{\circ}T{\circ}D$, and let $h$ denote the hypothesis chosen on Step~\ref{step:Oneexpmech}.
With probability at least $(1-\beta)$ we have that $\error_{S}(h)\leq\alpha$.
\end{claim}

\begin{proof}
Note that by Sauer's lemma, 
\begin{eqnarray*}
|H|&=&|\Pi_C(P)| \leq \left(\frac{e|P|}{\VC(C)}\right)^{\VC(C)}\\
&\leq&\left(\frac{e(|T|+|S|)}{\VC(C)}\right)^{\VC(C)}\\
&\leq&\left(\beta\exp(\frac{\alpha |S|}{2\VC(C)})\right)^{\VC(C)}\\
&\leq&\beta\exp(\frac{\alpha |S|}{2}).
\end{eqnarray*}

For every $(z_1,\ldots,z_\ell)\in \Pi_C(P)$, algorithm $LabelBoostProcedure$ adds to $H$ a hypothesis $f$ s.t.\ $\forall 1\leq j \leq \ell,\;f(p_j)=z_j$.
In particular, $H$ contains a hypothesis $f^*$ s.t.\ $\error_{S}(f^*)=0$. 
Hence, Proposition~\ref{prop:expMech} (properties of the exponential mechanism) ensures that the probability of the exponential mechanism choosing an $h$ s.t.\ $\error_{S}(h)>\alpha$ is at most
$$
|H|\cdot\exp(-\frac{\alpha |S|}{2})\leq\beta.
$$
\end{proof}

We next embed algorithm $LabelBoostProcedure$ in a wrapper algorithm, called $LabelBoost$, that iteratively applies $LabelBoostProcedure$ in order to enlarge the labeled portion of the database. Every such application deteriorates the privacy parameters, and hence, every iteration includes a sub-sampling step, which compensates for those privacy losses. In a nutshell, the learner $LabelBoost$ could be described as follows.
It starts by training on the given labeled data. In each step, a part of
the unlabeled points is labeled using the current hypothesis (previously labeled points are also relabeled); then the learner retrains using its own predictions as a (larger) labeled sample. Variants of this idea (known as self-training) have appeared in the literature for non-private learners (e.g., \cite{Scudder65,Fralick67,Agrawala70}). As we will see, in the context of {\em private} learners, this technique provably reduces the labeled sample complexity (while maintaining utility). 

\begin{algorithm}
\caption{$LabelBoost$}\label{alg:LabelBoost}
{\bf Setting:} Algorithm $\AAA$ with (labeled and unlabeled) sample complexity $n$.  \\  
{\bf Input:} An unlabeled database $D\in X^{90000n}$ and a labeled database $S\in (X\times\{0,1\})^m$.
\begin{enumerate}%[rightmargin=10pt,itemsep=1pt]

\item Set $i=1$.

\item \label{step:while} While $|S|<300n$:
\begin{enumerate}[label=\gray{\%},topsep=-10pt,leftmargin=*]
\item \gray{
$S$ denotes the currently labeled portion of the database. In each iteration, $|S|$ grows exponentially. The loop ends when $S$ is big enough s.t.\ we can apply the base learner $\AAA$ on $S$.
}
\end{enumerate}

\begin{enumerate}

\item Denote $\alpha_i=\frac{\alpha}{10\cdot2^i}$, and $\beta_i=\frac{\beta}{4\cdot2^i}$.

\item \label{step:addElements} Set $v{=}\min\hspace{-2pt}\left\{\hspace{-1pt}30000n \, , \, \beta_i \VC(C) e^{\frac{\alpha_i |S|}{200\VC(C)}}-|S|\hspace{-1pt}\right\}$. Let $T$ be the first $v$ elements of $D$, and remove $T$ from $D$. Fail if there are not enough elements in $D$.

\begin{enumerate}[label=\gray{\%},topsep=-10pt,leftmargin=*]
\item \gray{
We consider the input as a one database $(S{\circ}T{\circ}D)\in (X\times\{0,1,\bot\})^*$. The functionality of this step can, therefore, be viewed as changing the index in which $T$ ends and $D$ begins.
}
\end{enumerate}

\item \label{step:subsampling} Delete (permanently) $\frac{99}{100}|T|$ random entries from $T$, and $\frac{99}{100}|S|$ random entries from $S$.
\begin{enumerate}[label=\gray{\%},topsep=-10pt,leftmargin=*]
\item \gray{
Every iteration deteriorates the privacy parameters. We, therefore, boost the privacy guarantees using sub-sampling.
}
\end{enumerate}

\item \label{step:procedure} $S{\circ}T{\circ}D \leftarrow LabelBoostProcedure(S{\circ}T{\circ}D)$.
\begin{enumerate}[label=\gray{\%},topsep=-10pt,leftmargin=*]
\item \gray{
We use $LabelBoostProcedure$ to ``stretch'' the labeled portion of the database onto $T$.
}
\end{enumerate}

\item \label{step:uniteTS} Add every element of $T$ to $S$.

\item Set $i=i+1$.

\end{enumerate}

\item \label{step:subsamplingFinal} Delete $\frac{299}{300}|S|$ random entries from $S$.
\begin{enumerate}[label=\gray{\%},topsep=-10pt,leftmargin=*]
\item \gray{
Boosting privacy guarantees.
}
\end{enumerate}
\item \label{step:iidSampling} Let $S'$ denote the outcome of $|S|$ i.i.d.\ samples from $S$.
\begin{enumerate}[label=\gray{\%},topsep=-10pt,leftmargin=*]
\item \gray{
We apply $\AAA$ on $n$ i.i.d.\ samples from $S$. As $\AAA$ is a learner, it is required to output (w.h.p.) a hypothesis with small error on $S$.
}
\end{enumerate}
\item \label{step:AAA} Execute $\AAA$ on $S'$.

\end{enumerate}
\end{algorithm}

Before analyzing algorithm $LabelBoost$ we recall the sub-sampling technique from~\cite{KLNRS08, BBKN12}.

\begin{claim}[\cite{KLNRS08, BBKN12}]\label{claim:boostPrivacy}
Let $\AAA$ be an $(\epsilon^*,\delta)$-differentially private algorithm operating on databases of size $n$.
Fix $\epsilon\leq1$, and denote $t=\frac{n}{\epsilon}(3+\exp(\epsilon^*))$.
Construct an algorithm $\BBB$ that on input a database $D=(z_i)_{i=1}^t$ 
uniformly at random selects a subset $J\subseteq\{1,2,...,t\}$ of size $n$, and runs $\AAA$ on the multiset 
$D_J=(z_i)_{i\in J}$.
Then, $\BBB$ is $\left(\epsilon,\frac{4\epsilon}{3+\exp(\epsilon^*)}\delta\right)$-differentially private.
\end{claim}

\begin{rem}
In Claim~\ref{claim:boostPrivacy} we assume that $\AAA$ treats its input as a multiset. If this is not the case, then algorithm $\BBB$ should be modified to randomly shuffle the elements in $D_J$ before applying $\AAA$ on $D_j$.
\end{rem}

%remove

Claim~\ref{claim:boostPrivacy} boosts privacy by selecting random elements from the database and ignoring the rest of the database.
The intuition is simple: Fix two neighboring databases $D,D'$ differing (only) on their $i^\text{th}$ entry. 
If the $i^\text{th}$ entry is ignored (which happens with high probability), then the executions on $D$ and on $D'$ are the same (i.e., perfect privacy). 
Otherwise, $(\epsilon^*,\delta)$-privacy is preserved.

In algorithm $LabelBoost$ we apply the learner $\AAA$ on a database containing $n$ i.i.d.\ samples from the database $S$ (Step~\ref{step:iidSampling}). Consider two neighboring databases $D,D'$ differing on their $i^\text{th}$ entry. Unlike in Claim~\ref{claim:boostPrivacy}, the risk is that this entry will appear several times in the database on which $\AAA$ is executed. As the next claim states, the affects on the privacy guarantees are small. The intuition is that the probability of the $i^\text{th}$ entry appearing ``too many'' times is negligible.

\begin{claim}[\cite{BNSV14}]\label{claim:iidSampling}
Let $\epsilon\leq1$ and $\AAA$ be an $(\epsilon,\delta)$-differentially private algorithm operating on databases of size $n$. 
Construct an algorithm $\BBB$ that on input a database $D=(z_i)_{i=1}^n$ 
applies $\AAA$ on a database $D'$ containing $n$ i.i.d.\ samples from $D$.
Then, $\BBB$ is $\left(\ln(244),2467\delta\right)$-differentially private.
\end{claim}

%remove

We next prove the privacy properties of algorithm $LabelBoost$.

\begin{lem}\label{lemma:LabelBoostPrivacy}
If $\AAA$ is $(1,\delta)$-differentially private, then $LabelBoost$ is $(1,41\delta)$-differentially private.
\end{lem}

\begin{proof}
We think of the input of $LabelBoost$ as one database $B\in(X\times\{0,1,\bot\})^{90000n+m}$. Note that the number of iterations performed on neighboring databases is identical (determined by the parameters $\alpha,\beta,n,m$), and denote this number as $N$.
Throughout the execution, random elements from the input database are deleted (on Step~\ref{step:subsampling}). Note however, that the size of the database at any moment throughout the execution does not depend on the database content (determined by the parameters $\alpha,\beta,n,m$). We denote the size of the database at the beginning of the $i^\text{th}$ iteration as $n(i)$, e.g., $n(1)=90000n+m$.

Let $\LLL_t$ denote an algorithm similar to $LabelBoost$, except that only the last $t$ iterations are performed. The input of $\LLL_t$ is a database in $(X\times\{0,1,\bot\})^{n(N-t+1)}$.
We next show (by induction on $t$) that $\LLL_t$ is $(1,41\delta)$-differentially private.
To this end, note that an execution of $\LLL_0$ consists of sub-sampling (as in Claim~\ref{claim:boostPrivacy}), i.i.d.\ sampling (as in Claim~\ref{claim:iidSampling}), and applying the $(1,\delta)$-private algorithm $\AAA$. 
By Claim~\ref{claim:iidSampling}, steps~\ref{step:iidSampling}--\ref{step:AAA} preserve $(\ln(244),2476)$-differential privacy, and, hence, by Claim~\ref{claim:boostPrivacy}, we have that $\LLL_0$ is $(1,41\delta)$-differentially private.

Assume that $\LLL_{t-1}$ is $(1,41\delta)$-differentially private, and observe that $\LLL_t$ could be restated as an algorithm that first performs one iteration of algorithm $LabelBoost$ and then applies $\LLL_{t-1}$ on the databases $D,S$ as they are at the end of that iteration.
Now fix two neighboring databases $B_1,B_2$ and 
consider the execution of $\LLL_t$ on $B_1$ and on $B_2$.

Let $S_1^b,T_1^b,D_1^b$ and $S_2^b,T_2^b,D_2^b$ be the databases $S,T,D$ after Step~\ref{step:addElements} of the first iteration of $\LLL_t$ on $B_1$ and on $B_2$ (note that $B_1=S_1^b{\circ} T_1^b{\circ} D_1^b$ and $B_2=S_2^b{\circ} T_2^b{\circ} D_2^b$).
If $B_1$ and $B_2$ differ (only) on their last portion, denoted as $D_1^b,D_2^b$, then the execution of $\LLL_t$ on these neighboring inputs differs only in the execution of $\LLL_{t-1}$, and hence Inequality~(\ref{eqn:diffPrivDef}) (approximate differential privacy) follows from the privacy of $\LLL_{t-1}$. We, therefore, assume that $D_1^b=D_2^b$ (and that $S_1^b{\circ} T_1^b$ and $S_2^b{\circ} T_2^b$ differ in at most one entry).
Now, note that an execution of $\LLL_t$ consists of sub-sampling (as in Claim~\ref{claim:boostPrivacy}), applying algorithm $LabelBoostProcedure$ on the inputs, and executing the $(1,41\delta)$-private algorithm $\LLL_{t-1}$. 
By Lemma~\ref{lemma:TransformationPrivacy} (privacy properties of $LabelBoostProcedure$), the application of $\LLL_{t-1}$ on top of $LabelBoostProcedure$ preserves $(4,446\delta)$-differential privacy, and, hence, by Claim~\ref{claim:boostPrivacy} (sub-sampling), we have that $\LLL_t$ is $(1,41\delta)$-differentially private.
\end{proof}

Before proceeding with the utility analysis, we introduce to following notations.

\paragraph{Notation.}
Consider the $i^{\text{th}}$ iteration of $LabelBoost$.
We let $S_i^b,T_i^b$ and $S_i^c,T_i^c$ denote the elements $S,T$ as they are after Steps~\ref{step:addElements} and~\ref{step:subsampling},
and let $h_i$ denote the the hypothesis $h$ chosen in the execution of $LabelBoostProcedure$ in the $i^{\text{th}}$ iteration.

\begin{obs}\label{obs:LabelBoostEmpiricalError}
In every iteration $i$, with probability at least $(1-\beta_i)$ we have that $\error_{S_i^c}(h_i)\leq\alpha_i$.
\end{obs}

\begin{proof}
Follows from Claim~\ref{claim:LabelBoostProcedureUtility}.
\end{proof}

\begin{claim}
Let $LabelBoost$ be executed with a base learner with sample complexity $n$, and on databases $D,S$.
If $|D|\geq90000n$, then $LabelBoost$ never fails on Step~\ref{step:addElements}.
\end{claim}

\begin{proof}
Denote the number of iterations throughout the execution as $N$.
We need to show that $\sum_{i=1}^N T_i^b\leq90000n$.
Clearly, $|T_N^b|,|T_{N-1}^b|\leq30000n$.
Moreover, for every $1<i<N$ we have that $|T_i^b|\geq2|T_{i-1}^b|$. Hence,
$$\sum_{i=1}^N T_i^b \leq 30000n+30000n\sum_{i=0}^{\infty}\frac{1}{2^i}=90000n.\qedhere$$
\end{proof}

\begin{claim}\label{claim:LabelBoostDatabaseSize}
Fix $\alpha,\beta$.
Let $LabelBoost$ be executed on a base learner with sample complexity $n$, and on databases $D,S$, where $|D|\geq90000n$ and $|S|\geq\frac{96000}{\alpha}\VC(C)\log(\frac{2240}{\alpha\beta})$.
In every iteration $i$ 
$$|S_i^b|\geq\frac{4800}{\alpha_i}\VC(C)\log(\frac{14}{\alpha_i \beta_i}).$$
\end{claim}

\begin{proof}
The proof is by induction on $i$. Note that the base case (for $i=1$) trivially holds, and assume that the claim holds for $i-1$.
We have that
\begin{eqnarray*}
|S_i^b| &=& |S_{i-1}^c|+|T_{i-1}^c| = \frac{1}{100}(|S_{i-1}^b|+|T_{i-1}^b|)\\
&=& \frac{1}{100}\beta_{i-1}\VC(C)\exp\left(\frac{\alpha_{i-1}|S_{i-1}^b|}{200\VC(C)}\right)\\
&\geq& \frac{1}{100}\beta_{i-1}\VC(C)\exp\left(24 \log(\frac{14}{\alpha_{i-1} \beta_{i-1}})\right)\\
&\geq& \frac{1}{100}\beta_{i-1}\VC(C)\cdot\left(\frac{14}{\alpha_{i-1} \beta_{i-1}}\right)^{24}\\
&\geq& \frac{4800}{\alpha_i}\VC(C)\log(\frac{14}{\alpha_i \beta_i}).
\end{eqnarray*}
\end{proof}

\begin{rem}
The above analysis could easily be strengthen to show that $|S_i^b|$ grows as an exponentiation tower in $i$.
This implies that there are at most $O(\log^*n)$ iterations throughout the execution of $LabelBoost$ on a base learner $\AAA$ with sample complexity $n$.
\end{rem}

\begin{claim}\label{claim:LabelBoostGeneralization}
Let $LabelBoost$ be executed on databases $D,S$ containing i.i.d.\ samples from a fixed distribution $\mu$, where the examples in $S$ are labeled by some fixed target concept $c\in C$, and $|S|\geq\frac{96000}{\alpha}\VC(C)\log(\frac{2240}{\alpha\beta})$.
For every $i$, the probability that $\error_{\mu}(c,h_i)>10\sum_{j=1}^i\alpha_j$ is at most $2\sum_{j=1}^i\beta_j$.
\end{claim}

\begin{proof}
The proof is by induction on $i$. Note that for $i=1$ we have that $S_1^c$ contains $\frac{48}{\alpha_1}\VC(C)\log(\frac{14}{\alpha_1 \beta_1})$ i.i.d.\ samples from $\mu$ that are labeled by the target concept $c$. 
By Observation~\ref{obs:LabelBoostEmpiricalError}, with probability at least $(1-\beta_1)$, we have that $\error_{S_1^c}(h_1)\leq\alpha_1$.
In that case, Theorem~\ref{thm:VCconsistant} (the VC dimension bound) states that with probability at least $(1-\beta_1)$ it holds that $\error_{\mu}(c,h_1)\leq10\alpha_1$.

Now assume that the claim holds for $(i-1)$, and consider the $i^{\text{th}}$ iteration. Note that $S_i^c$ contains i.i.d.\ samples from $\mu$ that are labeled by $h_{i-1}$. 
Moreover, by Claim~\ref{claim:LabelBoostDatabaseSize}, we have that $|S_i^c|=\frac{1}{100}|S_i^b|\geq\frac{48}{\alpha_i}\VC(C)\log(\frac{14}{\alpha_i \beta_i})$.
By Observation~\ref{obs:LabelBoostEmpiricalError}, with probability at least $(1-\beta_i)$, we have that $\error_{S_i^c}(h_i)\leq\alpha_i$.
If that is the case, Theorem~\ref{thm:VCconsistant} states that with probability at least $(1-\beta_i)$ it holds that $\error_{\mu}(h_{i-1},h_i)\leq10\alpha_i$.
So, with probability at least $(1-2\beta_i)$ we have that $\error_{\mu}(h_{i-1},h_i)\leq10\alpha_i$.
Using the inductive assumption, the probability that $\error_{\mu}(c,h_i)\leq\error_{\mu}(c,h_{i-1})+\error_{\mu}(h_{i-1},h_i)\leq10\sum_{j=1}^i\alpha_j$ is at least $(1-2\sum_{j=1}^i\beta_j)$.
\end{proof}

\begin{lem}\label{lemma:LabelBoostUtility}
Fix $\alpha,\beta$.
Applying $LabelBoost$ on an $(\alpha,\beta,n,n)$-SSL for a class $C$ results in an $(11\alpha,2\beta,O(n),m)$-SSL for $C$, where $m=O(\frac{1}{\alpha}\VC(C)\log(\frac{1}{\alpha\beta}))$.
\end{lem}

\begin{proof}
Let $LabelBoost$ be executed on databases $D,S$ containing i.i.d.\ samples from a fixed distribution $\mu$, 
where $|D|\geq90000n$ and $|S|\geq\frac{96000}{\alpha}\VC(C)\log(\frac{2240}{\alpha\beta})$.
Moreover, assume that the examples in $S$ are labeled by some fixed target concept $c\in C$.

Consider the last iteration of Algorithm $LabelBoost$ (say $i=N$) on these inputs.
The intuition is that after the last iteration, when reaching Step~\ref{step:iidSampling}, the database $S$ is big enough s.t.\ $\AAA$ returns (w.h.p.) a hypothesis with small error on $S$. This hypothesis also has small generalization error as $S$ is labeled by $h_N$ which is close to the target concept (by Claim~\ref{claim:LabelBoostGeneralization}). 

Formally, let $S^3$ denote the database $S$ as it after Step~\ref{step:subsamplingFinal} of the execution, and let $h_{\text{fin}}$ denote the hypothesis returned by the base learner $\AAA$ on Step~\ref{step:AAA}.
By the while condition on Step~\ref{step:while}, we have that $|S^3|\geq n$.
Hence, by the utility guarantees of the base learner $\AAA$, with probability at least $(1-\beta)$ we have that $\error_{S^3}(h_{\text{fin}})\leq\alpha$.
As $|S^3|\geq\frac{1}{300}|S|\geq\frac{640}{\alpha}\VC(C)\log(\frac{4480}{\alpha\beta})$, and as $S^3$ contains i.i.d.\ samples from $\mu$ labeled by $h_N$, Theorem~\ref{thm:VCconsistant} states that with probability at least $(1-\frac{\beta}{2})$ it holds that $\error_{\mu}(h_{\text{fin}},h_N)\leq10\alpha$.
By Claim~\ref{claim:LabelBoostGeneralization}, with probability at least $(1-2\sum_{i=1}^N\beta_i)\geq(1-\frac{\beta}{2})$ it holds that $\error_{\mu}(c,h_N)\leq10\sum_{j=1}^N\alpha_i\leq\alpha$.
All in all (using the triangle inequality), with probability at least $(1-2\beta)$ we get that $\error_{\mu}(c,h_{\text{fin}})\leq11\alpha$.
\end{proof}

Combining Lemma~\ref{lemma:LabelBoostPrivacy} and Lemma~\ref{lemma:LabelBoostUtility} we get the following theorem.

\begin{thm}\label{thm:LabelBoost1}
Fix $\alpha,\beta,\delta$.
Applying $LabelBoost$ on an $(\alpha,\beta,\epsilon{=}1,\delta,n,n)$-PSSL for a class $C$ results in an $(11\alpha,2\beta,\epsilon{=}1,41\delta,O(n),m)$-PSSL for $C$, where $m=O(\frac{1}{\alpha}\VC(C)\log(\frac{1}{\alpha\beta}))$.
\end{thm}

Using Claim~\ref{claim:boostPrivacy} to boost the privacy guarantees of the learner resulting from Theorem~\ref{thm:LabelBoost1}, proves Theorem~\ref{thm:LabelBoost2}:

\begin{thm}\label{thm:LabelBoost2}
There exists a constant $\lambda$ such that:
For every $\alpha,\beta,\epsilon,\delta,n$, if there exists an $(\alpha,\beta,1,\delta,n,n)$-PSSL for a concept class $C$, then there exists an $(\lambda\alpha,\lambda\beta,\epsilon,\delta,O(\frac{n}{\epsilon}),m)$-PSSL for $C$, where $m=O(\frac{1}{\alpha\epsilon}\VC(C)\log(\frac{1}{\alpha\beta}))$.
\end{thm}

%----------------------

\begin{rem}
Let $\BBB$ be the learner resulting from applying $LabelBoost$ on a learner $\AAA$. Then (1)~If $\AAA$ preserves pure-privacy, then so does $\BBB$; and (2)~If $\AAA$ is a proper-learner, then so is $\BBB$.
\end{rem}

Algorithm $LabelBoost$ can also be used as an {\em agnostic} learner, where the target class $C$ is unknown, and the learner outputs a hypothesis out of a set $F\neq C$. Note that given a labeled sample, a consistent hypothesis might not exist in $F$. Minor changes in the proof of Theorem~\ref{thm:LabelBoost2} show 
the following theorem.

\begin{thm}
There exists a constant $\lambda$ such that:
For every $\alpha,\beta,\epsilon,\delta,n$, if there exists an $(\alpha,\beta,1,\delta,n,n)$-PSSL for a concept class $F$, then there exists an $(\lambda\alpha,\lambda\beta,\epsilon,\delta,O(\frac{n}{\epsilon}),m)$-agnostic-PSSL using $F$, where $m=O(\frac{1}{\alpha^2\epsilon}\VC(F)\log(\frac{1}{\alpha\beta}))$.
\end{thm}

To show the usefulness of Theorem~\ref{thm:LabelBoost2}, we consider (a discrete version of) the class of all axis-aligned rectangles (or hyperrectangles) in $\ell$ dimensions.
Formally,
%\vspace{-0.05in}\begin{defn}
%Let
let 
$X_d^{\ell}=(\{0,1\}^d)^{\ell}$ denote a discrete ${\ell}$-dimensional domain, in which every axis consists of $2^d$ points.
For every $\vec{a}=(a_1,\ldots,a_{\ell}),\vec{b}=(b_1,\ldots,b_{\ell})\in X_d^{\ell}$ define the concept $c_{[\vec{a},\vec{b}]}:X_d^{\ell}\rightarrow\{0,1\}$ where $c_{[\vec{a},\vec{b}]}(\vec{x})=1$ if and only if for every $1\leq i\leq {\ell}$ it holds that $a_i\leq x_i\leq b_i$. Define the concept class of all axis-aligned rectangles over $X^{\ell}_d$ as $\rectangle_d^{\ell}=\{ c_{[\vec{a},\vec{b}]} \}_{\vec{a},\vec{b}\in X_d^{\ell}}$.
%\end{defn}
The VC dimension of this class is $2{\ell}$, and, thus, it can be learned non-privately with (labeled and unlabeled) sample complexity $O_{\alpha,\beta}({\ell})$.
The best currently known private PAC learner for this class~\cite{BNS13b} has (labeled and unlabeled) sample complexity $\widetilde{O}_{\alpha,\beta,\epsilon,\delta}({\ell}^3 \cdot 8^{\log^*d})$. Using $LabelBoost$ with the construction of~\cite{BNS13b} reduces the labeled sample complexity while maintaining the unlabeled sample complexity. 

%\vspace*{-0.05in}
\begin{cor}
There exists a private semi-supervised learner for $\rectangle_d^{\ell}$ with unlabeled sample complexity $\widetilde{O}_{\alpha,\beta,\epsilon,\delta}({\ell}^3 \cdot 8^{\log^*d})$ and labeled sample complexity $O_{\alpha,\beta,\epsilon}(\ell)$. The learner is efficient (runs in polynomial time) whenever the dimension $\ell$ is small enough (roughly, $\ell \leq \log^{\frac{1}{3}} d$).
\end{cor}
%\vspace*{-0.05in}

%Notice that t
The {\em labeled} sample complexity in Theorem~\ref{thm:LabelBoost2} has no dependency in $\delta$.\footnote{The unlabeled sample complexity depends on $\delta$ as $n$ depends on $\delta$.}
It would be helpful if we could also reduce the dependency on $\epsilon$. As we will later see, this can be achieved in the active learning model.

%\subsection{Comparing $LabelBoostProcedure$ and $GenericLearner$}
\paragraph{$\mathbf{LabelBoost}$ vs. $\mathbf{GenericLearner}$.}
While both constructions result in learners with labeled sample complexity proportional to the VC dimension, 
they differ on their unlabeled sample complexity.

Recall the generic construction of Kasiviswanathan et al.~\cite{KLNRS08} for private PAC learners, in which the (labeled and unlabeled) sample complexity 
is logarithmic in the size of the target concept class $C$ (better constructions are known for many specific cases). Using Algorithm $LabelBoost$ with their generic construction results in a private semi-supervised learner with unlabeled sample complexity (roughly) $\log|C|$, which is better than the bound achieved by $GenericLearner$ (whose unlabeled sample complexity is $O(\log|X|\cdot\VC(C))$).
In cases where a sample-efficient private-PAC learner is known, applying $LabelBoost$ would give even better bounds.

Another difference is that (a direct use of) $GenericLearner$ only yields pure-private proper-learners, whereas $LabelBoost$ could be applied to every private learner
(proper or improper, preserving pure or approximated privacy).
To emphasize this difference, recall that the sample complexity of pure-private improper-PAC-learners is characterized by the Representation Dimension~\cite{BNS13}. 
%That is, for every concept class $C$ there exists a private PAC learner with (labeled and unlabeled) sample complexity $\RepDim(C)$. Moreover, Beimel et al.~\cite{BNS13} showed that there exist classes $C$ for which $\RepDim(C)\ll\log|C|$. For such classes the following corollary gives better bounds than $GenericLearner$ (the corollary follows from using $LabelBoostProcedure$ on a learner with sample complexity $\RepDim(C)$).

%\vspace*{-0.05in}
\begin{cor}
For every concept class $C$ there is a pure-private semi-supervised improper-learner with labeled sample complexity $O_{\alpha,\beta,\epsilon}(\VC(C))$ and unlabeled sample complexity $O_{\alpha,\beta,\epsilon}(\RepDim(C)) $.
\end{cor}
%\vspace*{-0.05in}

\section{Private Active Learners} \label{sec:privActive}

Semi-supervised learners are a subset of the larger family of active learners. Such learners can adaptively request to reveal the labels of specific examples. An active learner is given access to a pool of $n$ unlabeled examples, and adaptively chooses to label $m$ examples.

%\vspace{-0.05in}
\begin{defn}[Active Learning~\cite{ActiveLearning}]\label{def:AL}
Let $C$ be a concept class over a domain $X$.
Let $\AAA$ be an interactive (stateful) algorithm that holds an initial input database $D=(x_i)_{i=1}^n\in(X)^n$.
For at most $m$ rounds, algorithm $\AAA$ outputs an index $i\in\{1,2,\ldots,n\}$ and receives an answer $y_i\in\{0,1\}$.
Afterwards, algorithm $\AAA$ outputs a hypothesis $h$, and terminates.

Algorithm $\AAA$ is an {\em $(\alpha,\beta,n,m)$-AL (Active learner)} for $C$ if for all concepts $c\in C$ and all distributions $\mu$ on $X$:
If $\AAA$ is initiated on an input $D=(x_i)_{i=1}^n$, where each $x_i$ is drawn i.i.d.\ from $\mu$, and if every index $i$ queried by $\AAA$ is answered by $y_i=c(x_i)$, then algorithm $\AAA$ outputs a hypothesis $h$ satisfying
$\Pr[\error_{\mu}(c,h)  \leq \alpha] \geq 1-\beta.$
The probability is taken over the random choice of
the samples from $\mu$ and the coin tosses of the learner $\AAA$.
\end{defn}

\begin{rem}
In the standard definition of active learners, the learners specify examples by their value (whereas in Definition~\ref{def:AL} the learner queries the labels of examples by their index). E.g., if $x_5=x_9=p$ then instead of asking for the label of $p$, algorithm $\AAA$ asks for the label example 5 (or 9).
This deviation from the standard definition is because when privacy is introduced, every entry in $D$ corresponds to a single individual, and can be changed arbitrarily (and regardless of the other entries).
\end{rem}

%We now present the definition for private active learners from~\cite{BF13}.
%\vspace*{-0.1in}

%\vspace{-0.05in}\begin{defn}[Active Oracle]
%Let $C$ be a concept class over a domain $X$, and let $\AAA$ be an $(\alpha,\beta,n,m)$-active-learner for $C$. 
%We denote by $\AAA^{\OOO}$ an algorithm that on input a database $S=(x_i,y_i)_{i=1}^n\in(X\times\{0,1\})^n$ runs $\AAA$ on %$D=(x_i)_{i=1}^n$ and answers every request for the label of the $i^{\text th}$ entry in $D$ using $y_i$.
%The output of $\AAA^{\OOO}$ is the hypothesis $h$ returned by $\AAA$ at the end of its execution.
%\end{defn}
%\vspace{-0.05in}

%\vspace{-0.05in}
\begin{defn}[Private Active Learner~\cite{BF13}] \label{def:PAL}
An algorithm $\AAA$ is an {\em $(\alpha,\beta,\epsilon,\delta,n,m)$-PAL} (Private Active Learner) for a concept class $C$ if Algorithm $\AAA$ is an $(\alpha,\beta,n,m)$-active learner for $C$ and 
$\AAA$ is $(\epsilon,\delta)$-differentially private, where in the definition of privacy we consider the input of $\AAA$ to be a fully labeled sample $S=(x_i,y_i)_{i=1}^n\in(X\times\{0,1\})^n$ (and limit the number of labels $y_i$ it can access to $m$). 
\end{defn}
%\vspace{-0.05in}

Note that the queries that an active learner makes depend on individuals' data. Hence, if the indices that are queried are exposed, they may breach privacy. An example of how such an exposure may occur is a medical research of a new disease -- a hospital may posses background information about individuals and hence can access a large pool of unlabeled examples, but to label an example an actual medical test is needed. Partial information about the labeling queries would hence be leaked to the tested individuals. More information about the queries may be leaked to an observer of the testing site. The following definition remedies this potential breach of privacy.

%\vspace{-0.05in}\begin{defn}
%Let $C$ be a concept class over a domain $X$, and let $\AAA$ be an $(\alpha,\beta,n,m)$-active-learner for $C$. 
%We denote by $\AAA^{\OOO}_L$ an algorithm that on input a database $S=(x_i,y_i)_{i=1}^n\in(X\times\{0,1\})^n$ runs $\AAA$ on %$D=(x_i)_{i=1}^n$ and answers every request for the label of the $i^{\text th}$ entry in $D$ using $y_i$. Let $h$ denote %$\AAA$'s output, and let $L=(\ell_i)_{i=1}^m\in\{1,2,\ldots,n\}^m$ denote the ordered sequence of indices of $\AAA$'s %queries throughout the simulation. %For example, if the fifth query made by $\AAA$ was for the label of $7^{\text th}$ entry %of $D$, then $\ell_5=7$.
%The output of $\AAA^{\OOO}_L$ is the pair $(h,L)$.
%\end{defn}

%\vspace*{-0.05in}
\begin{defn}\label{def:transcriptPrivate}
We define the transcript in an execution of an active learner $\AAA$ as the ordered sequence $L=(\ell_i)_{i=1}^m\in\{1,2,\ldots,n\}^m$ of indices that $\AAA$ outputs throughout the execution.
We say that a learner $\AAA$ is $(\epsilon,\delta)$-transcript-differentially private
if the algorithm whose input is the labeled sample and whose output is the output of $\AAA$ together with 
the transcript of the execution is $(\epsilon,\delta)$-differentially private.
An algorithm $\AAA$ is an {\em $(\alpha,\beta,\epsilon,\delta,n,m)$-TPAL (transcript-private active-learner)} for a concept class $C$ if Algorithm $\AAA$ is an $(\alpha,\beta,n,m)$-Active learner for $C$ and $\AAA$ is $(\epsilon,\delta)$-transcript-differentially private.
\end{defn}
%\vspace{-0.05in}

Recall that a semi-supervised learner has no control over 
which of its examples are labeled, and the indices of the labeled examples are publicly known.
Hence, a private semi-supervised learner is, in particular, a transcript-private active learner.

\begin{thm}
If $\AAA$ is an $(\alpha,\beta,\epsilon,\delta,n,m)$-PSSL, then $\AAA$ is an $(\alpha,\beta,\epsilon,\delta,n,m)$-TPAL.
\end{thm}

In particular, our algorithms from Sections~\ref{sec:semiSuper} and~\ref{sec:boost} satisfy Definition~\ref{def:transcriptPrivate}, suggesting that the strong privacy guarantees of Definition~\ref{def:transcriptPrivate} are achievable. 
However, as we will now see, this comes with a price.
The work on (non-private) active learning has mainly focused on reducing the dependency of the labeled sample complexity in $\alpha$ (the approximation parameter).
The classic result in this regime states that the labeled sample complexity of learning $\thresh_d$ without privacy is $O(\log(\frac{1}{\alpha}))$, exhibiting an exponential improvement over the $\Omega(\frac{1}{\alpha})$ labeled sample complexity in the non-active model.
As the next theorem states, the labeled sample complexity of every transcript-private active-learner for $\thresh_d$ is lower bounded by $\Omega(\frac{1}{\alpha})$.

\begin{thm}\label{thm:TPALlowerBound}
Let $\alpha\leq\frac{1}{9}$ and $\beta\leq\frac{1}{4}$.
In every $(\alpha,\beta,\epsilon,\delta,n,m)$-TPAL for $\thresh_d$ the labeled sample complexity satisfies $m=\Omega\left(\frac{1}{\alpha}\right)$.
\end{thm}

\begin{proof}
Let $\AAA$ be an $(\alpha,\beta,\epsilon,\delta,n,m)$-TPAL for $\thresh_d$ with $\alpha\leq1/9$ and $\beta\leq1/4$.
Without loss of generality, we can assume that $n\geq\frac{100}{\alpha^2}\ln(\frac{1}{\alpha\beta})$ (since $\AAA$ can ignore part of the sample).
Denote $B=\{1,2,\ldots,8\alpha 2^d\}\subseteq X_d$, and consider the following thought experiment for randomly generating a labeled sample of size $n$.

\begin{center}
\noindent\fbox{
\parbox{0.95\textwidth}{
\hspace{0px}\vspace{-5px}
\begin{enumerate}[rightmargin=10pt,itemsep=1pt]
	\item Let $D=(x_1, x_2, \ldots, x_n)$ denote the outcome of $n$ uniform i.i.d.\ draws from $X_d$.
	\item Uniformly at random choose $t\in B$, and let $c_t\in\thresh_d$ be s.t.\ $c_t(x)=1$ iff $x<t$.
	\item Return $S=(x_i,c_t(x_i))_{i=1}^n$.
\end{enumerate}
}}
\end{center}

The above process induces a distribution on labeled samples of size $n$, denoted as $\PPP$.
Let $S\sim\PPP$, and consider the execution of $\AAA$ on $S$. Recall that $\AAA$ operates on the unlabeled portion of $S$ and actively queries for labels. 
Let $b$ denote the the number of elements from $B$ in the database $S$.
Standard arguments in learning theory (see Theorem~\ref{thm:generalization}) state that with all but probability $\beta\leq\frac{1}{4}$ it holds that $7\alpha n\leq b\leq9\alpha n$. We continue with the proof assuming that this is the case.
We first show that $\AAA$ must (w.h.p.) ask for the label of at least one example in $B$.
To this end, note that even given the labels of all $x\notin B$, the target concept is distributed uniformly on $B$, and the probability that $\AAA$ fails to output an $\alpha$-good hypothesis is at least $\frac{3}{4}$. Hence,

\begin{eqnarray*}
\beta&\geq&\Pr_{S,\AAA}[\AAA \text{ fails}]\\
&\geq&\Pr_{S,\AAA}\left[  \begin{array}{c}
	\AAA \text{ does not ask for the label}\\
	\text{of any point in } B \text{ and fails}
\end{array} \right]\\
&=&\Pr_{S,\AAA}\left[  \begin{array}{c}
	\AAA \text{ does not ask for the}\\
	\text{label of any point in } B 
\end{array} \right]
\cdot
\Pr_{S,\AAA}\left[  \AAA \text{ fails} \middle\vert
\begin{array}{c}
	\AAA \text{ does not ask for the}\\
	\text{label of any point in } B
\end{array} \right]\\
&\geq&\Pr_{S,\AAA}\left[  \begin{array}{c}
	\AAA \text{ does not ask for the}\\
	\text{label of any point in } B
\end{array} \right]
\cdot
\frac{3}{4}\\
&\geq&\Pr_{S}[b\leq9\alpha n]\cdot\Pr_{S,\AAA}\left[  \begin{array}{c}
	\AAA \text{ does not ask for the}\\
	\text{label of any point in } B
\end{array} \middle\vert b\leq9\alpha n \right]
\cdot
\frac{3}{4}\\
&\geq&\frac{9}{16}\cdot\Pr_{S,\AAA}\left[  \begin{array}{c}
	\AAA \text{ does not ask for the}\\
	\text{label of any point in } B
\end{array} \middle\vert b\leq9\alpha n \right].
\end{eqnarray*}
Thus, assuming that $b\leq9\alpha n$, the probability that $\AAA$ asks for the label of a point in $B$ is at least $(1-\frac{16}{9}\beta)$.
Now choose a random $x^*$ from $S$ s.t.\ $x^*\in B$. 
Note that
\begin{eqnarray*}
\Pr_{S,x^*,\AAA}\left[ \AAA(S) \text{ asks for the label of } x^* \right]
&\geq&\Pr_{S}[b\leq9\alpha n]\cdot\Pr_{S,x^*,\AAA}\left[ 
\begin{array}{c}
	\AAA(S) \text{ asks for}\\
	\text{the label of } x^*
\end{array}
\middle\vert
b\leq9\alpha n\right]\\
&\geq&(1-\beta)\cdot\frac{(1-\frac{16}{9}\beta)}{9\alpha n}\\
&\geq&\frac{1-\frac{25}{9}\beta}{9\alpha n}.
\end{eqnarray*}

Choose a random $\hat{x}$ from $S$ (uniformly), and construct a labeled sample $S'$ by swapping the entries $(x^*,c(x^*))$ and $(\hat{x},c(\hat{x}))$ in $S$.
Note that $S'$ is also distributed according to $\PPP$, and that $\hat{x}$ is a uniformly random element of $S'$.
Therefore,
$$
\Pr_{S,x^*,\hat{x},\AAA}\left[ \AAA(S') \text{ asks for the label of } \hat{x} \right]\leq\frac{m}{n}.
$$
As $S$ and $S'$ differ in at most 2 entries, differential privacy states that
\begin{eqnarray*}
\frac{m}{n} &\geq& \Pr_{S,x^*,\hat{x},\AAA}\left[ \AAA(S') \text{ asks for the label of } \hat{x} \right]\\
&=& \sum_{S,x^*,\hat{x}}\Pr[S,x^*,\hat{x}]\cdot\Pr_{\AAA}\left[ \AAA(S') \text{ asks for the label of } \hat{x} \right]\\
&\geq& \sum_{S,x^*,\hat{x}}\Pr[S,x^*,\hat{x}] \; e^{-2\epsilon} \; \Pr_{\AAA}\left[ \AAA(S) \text{ asks for the label of } x^* \right]-\delta(1{+}e^{-\epsilon})\\
&=& e^{-2\epsilon}\cdot\Pr_{S,x^*,\AAA}\left[ \AAA(S) \text{ asks for the label of } x^* \right]-\delta(1+e^{-\epsilon})\\
&\geq& e^{-2\epsilon}\cdot\frac{1-\frac{25}{9}\beta}{9\alpha n}-\delta(1+e^{-\epsilon}).
\end{eqnarray*}
Solving for $m$, this yields $m=\Omega(\frac{1}{\alpha})$.
\end{proof}

The private active learners presented in~\cite{BF13} as well as the algorithm described in the next section only satisfy the weaker Definition~\ref{def:PAL}.

\subsection{Removing the Dependency on the Privacy Parameters}

We next show how to transform a semi-supervised private learner $\AAA$ into an active learner $\BBB$ with better privacy guarantees without increasing the labeled sample complexity.
Algorithm $\BBB$, on input an unlabeled database $D$, randomly chooses a subset of the inputs $D'\subseteq D$ and asks for the labels of the examples in $D'$ (denote the resulting labeled database as $S$). Algorithm $\BBB$ then applies $\AAA$ on $D,S$.
As the next claim states, this eliminates the $\frac{1}{\epsilon}$ factor from the labeled sample complexity as the (perhaps adversarial) choice for the input database is independent of the queries chosen.

\begin{claim}\label{claim:activeBoostPrivacy}
If there exists an $(\alpha,\beta,\epsilon^*,\delta,n,m)$-PSSL for a concept class $C$, then for every $\epsilon$ there exists an $\left(\alpha,\beta,\epsilon,\frac{7+e^{\epsilon^*}}{3+e^{2\epsilon^*}}\epsilon\delta,t,m\right)$-PAL (private active learner) for $C$, where $t=\frac{n}{\epsilon}(3+\exp(2\epsilon^*))$.
\end{claim}

\begin{algorithm}
\caption{$SubSampling$}\label{alg:algWrapper}
{\bf Inputs:} Base learner $\AAA$, privacy parameters $\epsilon^*,\epsilon$, and a database $D=(x_i)_{i=1}^t$ of $t$ unlabeled examples.
\begin{enumerate}[rightmargin=10pt,itemsep=1pt]

\item Uniformly at random select a subset $J\subseteq\{1,2,...,t\}$ of size $n$, and let $K\subseteq J$ denote the smallest $m$ indices in $J$.

\item Request the label of every index $i\in K$, and let $\{y_i\;:\;i\in K\}$ denote the received answers. 

\item Run $\AAA$ an the multiset $D_J=\{ (x_i,\bot) : i\in J\setminus K \}\cup\{ (x_i,y_i) : i \in K \}$.

\end{enumerate}
\end{algorithm}

\begin{proof}
The proof is via the construction of Algorithm $SubSampling$ (Algorithm~\ref{alg:algWrapper}).
The utility analysis is straight forward. Fix a target concept $c$ and a distribution $\mu$. Assume that $D$ contains $t$ i.i.d.\ samples from $\mu$ and that every query on an index $i$ is answered by $c(x_i)$. Therefore, algorithm $\AAA$ is executed on a multiset $D_J$ containing $n$ i.i.d.\ samples from $\mu$ where $m$ of those samples are labeled by $c$. By the utility properties of $\AAA$, an $\alpha$-good hypothesis is returned with probability at least $(1-\beta)$.

For the privacy analysis, fix two neighboring databases $S,S'\in(X\times\{0,1\})^t$ differing on their $i^{\text th}$ entry, and let $D,D'\in X^t$ denote the restriction of those two databases to $X$ (that is, $D$ contains an entry $x$ for every entry $(x,y)$ in $S$).
Consider an execution of $SubSampling$ on $D$ (and on $D'$), and let $J\subseteq\{1,\ldots,t\}$ denote the random subset of size $n$ chosen on Step~1. Moreover, and let $D_J$ denote the multiset on which $\AAA$ in executed.

Since $S$ and $S'$ differ in just the $i^{\text th}$ entry, for any set of outcomes $F$ it holds that $\Pr[\AAA(D_J) \in F | i \not\in J] = \Pr[\AAA(D'_J) \in F | i \notin J]$. When $i\in J$ we have that
\begin{eqnarray*}
\Pr[SubSampling(D) \in F \wedge i\in J]
&=&\sum_{\begin{array}{c} {\scriptstyle R \subseteq [t]\setminus \{i\}}\\{\scriptstyle |R|=n-1}\end{array}}\hspace{-7pt}\Pr[J=R\cup\{i\}]\cdot\Pr[\AAA(D_J)\in F | J = R\cup \{i\}].
\end{eqnarray*}
Note that for every choice of $R \subseteq [t]\setminus \{i\}$ s.t.\ $|R|=(n-1)$, there are exactly $(t-n)$ choices for $Q \subseteq [t]\setminus \{i\}$ s.t.\ $|Q|=n$ and $R\subseteq Q$. Hence,
%\small
\begin{eqnarray*}
\Pr[SubSampling(D) \in F \wedge i\in J] &=&
\sum_{\begin{array}{c} {\scriptstyle R \subseteq [t]\setminus \{i\}}\\{\scriptstyle |R|=n-1}\end{array}} \hspace{-10pt} \frac{1}{t-n} \hspace{-10pt} \sum_{\begin{array}{c} {\scriptstyle Q \subseteq [t]\setminus \{i\}}\\{\scriptstyle |Q|=n}\\{\scriptstyle R\subseteq Q}\end{array}} \hspace{-13pt} \Pr[J{=}R{\cup}\{i\}]{\cdot}\Pr[\AAA(D_J){\in} F | J {=} R{\cup} \{i\}] \\
&\leq& \sum_{\begin{array}{c} \vspace{-6pt}\\{\scriptstyle R \subseteq [t]\setminus \{i\}}\\{\scriptstyle |R|=n-1}\end{array}} \hspace{-15pt} \frac{1}{t-n} \hspace{-15pt} \sum_{\begin{array}{c} \vspace{-6pt}\\{\scriptstyle Q \subseteq [t]\setminus \{i\}}\\{\scriptstyle |Q|=n}\\{\scriptstyle R\subseteq Q}\end{array}} \hspace{-15pt} \Pr[J{=}Q]\left(\e^{2\epsilon^*}\Pr[\AAA(D_J){\in} F | J {=} Q]{+}\delta{+}\delta e^{\epsilon^*}\right).
\end{eqnarray*}
%\normalsize
For the last inequality, note that $D_Q$ and $D_{R\cup\{i\}}$ differ in at most two entries, as they differ in one unlabeled example, and possibly one other example that is labeled in one multiset and unlabeled on the other.
Now note that every choice of $Q$ will appear in the above sum exactly $n$ times (as the number of choices for appropriate $R$'s s.t.\ $R\subseteq Q$). Hence,
\begin{eqnarray*}
\Pr\left[\{SubSampling(D) \in F \} \wedge \{ i\in J \}\right]
&\leq&\frac{n}{t-n}\hspace{-10pt}\sum_{\begin{array}{c} {\scriptstyle Q \subseteq [t]\setminus \{i\}}\\{\scriptstyle |Q|=n}\end{array}}\hspace{-12pt}\Pr[J{=}Q]\left(\e^{2\epsilon^*}\Pr[\AAA(D_J){\in} F | J {=} Q]{+}\delta{+}\delta e^{\epsilon^*}\right)\\
&=&\frac{n}{t-n}\cdot\Pr[i\notin J]\left(
e^{2\epsilon^*}\Pr[\AAA(D_J)\in F | i\notin J] {+} \delta {+} \delta e^{\epsilon^*}
\right)\\
&=&\frac{n}{t}e^{2\epsilon^*}\cdot\Pr[\AAA(D_J)\in F | i\notin J] + \frac{n}{t}(1+e^{\epsilon^*})\delta\\
&=&\frac{n}{t}e^{2\epsilon^*}\cdot\Pr[\AAA(D'_J)\in F | i\notin J] + \frac{n}{t}(1+e^{\epsilon^*})\delta.
\end{eqnarray*}
Therefore,
%\small
\begin{eqnarray*}
\Pr[SubSampling(D) \in F]
&=&\Pr\left[\{SubSampling {\in} F\} {\wedge} \{i{\in} J\}\right] {+} \Pr[i{\notin} J]{\cdot}\Pr[\AAA(D'_J){\in} F | i{\notin} J]\\
&\leq&\left(\frac{n}{t}e^{2\epsilon^*}+\frac{t-n}{t}\right)\cdot\Pr[\AAA(D'_J)\in F | i\notin J] + \frac{n}{t}(1+e^{\epsilon^*})\delta.
\end{eqnarray*}
%\normalsize
Similar arguments show that
\begin{eqnarray*}
\Pr[SubSampling(D') \in F]
&\geq& \left(\frac{n}{t}e^{-2\epsilon^*}+\frac{t-n}{t}\right)\cdot\Pr[\AAA(D'_J)\in F | i\notin J] - \frac{n}{t}2\delta.
\end{eqnarray*}
For $t\geq\frac{n}{\epsilon}(3+\exp(2\epsilon^*))$, this yields
\begin{eqnarray*}
&&\hspace{-25pt}\Pr[SubSampling(D) \in F]\\
&&\hspace{-20pt} \leq e^{\epsilon}\cdot\Pr[SubSampling(D') \in F]+\frac{7+e^{\epsilon^*}}{3+e^{2\epsilon^*}}\epsilon\delta.
\end{eqnarray*}
\end{proof}

The transformation of Claim~\ref{claim:activeBoostPrivacy} preserves the efficiency of the base (non-active) learner. Hence, a given (efficient) non-active private learner could always be transformed into an (efficient) active private learner whose labeled sample complexity does not depend on $\epsilon$.
Applying Claim~\ref{claim:activeBoostPrivacy} to the learner from Theorem~\ref{thm:LabelBoost2} result in the following theorem, showing that the labeled sample complexity of private active learners has no dependency in the privacy parameters $\epsilon$ and $\delta$.

\begin{thm}
There exists a constant $\lambda$ such that:
For every $\alpha,\beta,\epsilon,\delta,n$, if there exists an $(\alpha,\beta,1,\delta,n,n)$-PSSL for a concept class $C$, then there exists an $(\lambda\alpha,\lambda\beta,\epsilon,\delta,O(\frac{n}{\epsilon}),m)$-PAL for $C$, where $m=O(\frac{1}{\alpha}\VC(C)\log(\frac{1}{\alpha\beta}))$.
\end{thm}

\paragraph{{\bf Acknowledgments.}} We thank Aryeh Kontorovich, Adam Smith, and Salil Vadhan for helpful discussions of ideas in this work.

\newpage

\bibliographystyle{plain}

\newpage

\appendix

\section{Some Differentially Private Mechanisms} \label{sec:dp_mech}

\subsection{The Exponential Mechanism~\cite{MT07}}

We next describe the exponential mechanism of McSherry and Talwar~\cite{MT07}. We present its private learning variant; however, it can be used in more general scenarios. The goal here is to chooses a hypothesis $h\in H$ approximately minimizing the empirical error. The choice is probabilistic, where the probability mass that is assigned to each hypothesis decreases exponentially with its empirical error.

\begin{algorithm}
\caption{Exponential Mechanism}
{\bf Inputs:} Privacy parameter $\epsilon$, finite hypothesis class $H$, and $m$ labeled examples $S=(x_i,y_i)_{i=1}^m$.
\begin{enumerate}[rightmargin=10pt,itemsep=1pt]
  \item $\forall h\in H$ define $q(S,h)=|\{i:h(x_i)=y_i\}|$.
	\item Randomly choose $h \in H$ with probability
	$\frac{\exp\left(\epsilon \cdot q(S,h) /2 \right)}{\sum_{f\in H}\exp\left(\epsilon \cdot q(S,f) /2 \right)}$.
	\item Output $h$.
\end{enumerate}
\end{algorithm}
\begin{prop}[The Exponential Mechanism]\label{prop:expMech}
(i) The exponential mechanism is $\epsilon$-differentially private. (ii)
Let $\hat e\triangleq\min_{f\in H}\{\error_S(f)\}$. For every $\Delta>0$, the probability that the exponential mechanism outputs a hypothesis $h$ such that $\error_S(h)>\hat e + \Delta$ is at most $|H| \cdot \exp(-\epsilon \Delta m /2)$.
\end{prop}

\subsection{Data Sanitization}
Given a database $S=(x_1,\ldots,x_m)$ containing elements from some domain $X$, the goal of data sanitization is to output (while preserving differential privacy) another database $\hat{S}$ that is in some sense similar to $S$. This returned database $\hat{S}$ is called a {\em sanitized} database, and the algorithm computing $\hat{S}$ is called a {\em sanitizer}.

For a concept $c:X\rightarrow\{0,1\}$ define $Q_c:X^*\rightarrow[0,1]$ as 
$Q_c(\db) = \frac{1}{|\db|}\cdot \Big|\{i \,:\,  c(x_i) =1\} \Big|.$
That is, $Q_c(\db)$ is the fraction of the entries in $\db$ that satisfy $c$. A sanitizer for a concept class $C$ is a differentially private algorithm that given a database $\db$ outputs a database $\hat{\db}$ s.t.\ $Q_c(\db) \approx Q_c(\hat{\db})$ for every $c\in C$.

%\vspace{-0.05in}
\begin{defn}[Sanitization \cite{BLR08full}]
Let $C$ be a class of concepts mapping $X$ to $\{0,1\}$. Let $\AAA$ be an algorithm that on an input database $S\in X^*$ outputs another database $\hat{S}\in X^*$. Algorithm $\AAA$ is an $(\alpha,\beta,\epsilon,\delta,m)$-sanitizer for predicates in the class $C$, if
\begin{enumerate}
\item $\AAA$ is $(\epsilon,\delta)$-differentially private;
\item For every input $S\in X^m$,
$$\Pr\limits_{\AAA}\left[  \exists c\in C \text{ s.t.\ } |Q_c(S)-Q_c(\hat{S})|>\alpha \right]\leq \beta.$$
\end{enumerate}
The probability is over the coin tosses of algorithm $\AAA$. As before, when $\delta{=}0$ (pure privacy) we omit it from the set of parameters.
\end{defn}

%Note that without the privacy requirements sanitization is a trivial task as it is possible to simply output the input database $\db$. 

\begin{thm}[Blum et al.~\cite{BLR08full}]\label{thm:BlumUp}
For any class of predicates $C$ over a domain $X$, and any parameters $\alpha,\beta,\epsilon$,
there exists an $(\alpha,\beta,\epsilon,m)$-sanitizer for $C$, where the size of the database $m$ satisfies:
$$m = O\left(\frac{\log|X|\cdot \VC(C)\cdot\log(1/\alpha)}{\alpha^3\epsilon}+\frac{\log(1/\beta)}{\epsilon\alpha}\right).$$
The returned sanitized database contains $O(\frac{\VC(C)}{\alpha^2}\log(\frac{1}{\alpha}))$ elements.
\end{thm}

\section{The Vapnik-Chervonenkis Dimension}\label{sec:VC}

The Vapnik-Chervonenkis (VC) Dimension is a combinatorial measure of concept classes that characterizes the sample size of PAC learners. Let $C$ be a concept class over a domain $X$, and let $B=\{b_1,\ldots,b_\ell\}\subseteq X$. The set of all dichotomies on $B$ that are realized by $C$ is 
$\Pi_C(B)=\Big\{(c(b_1),\ldots,c(b_\ell)):c\in C\Big\}$. A set $B\subseteq X$ is {\em shattered} by $C$ if $C$ realizes all possible dichotomies over $B$, i.e., $\Pi_C(B)=\{0,1\}^{|B|}$. 

%\vspace{-0.05in}
\begin{defn}[VC-Dimension~\cite{VC}]
The $\VC(C)$ is the cardinality of the largest set $B\subseteq X$ shattered by $C$. If arbitrarily large finite sets can be shattered by $C$, then $\VC(C)=\infty$.
\end{defn}

Sauer's lemma bounds the cardinality of $\Pi_C(B)$ in terms of $\VC(C)$ and $|B|$.

\begin{thm}[\cite{Sauer}]\label{thm:sauer}
Let $C$ be a concept class over a domain $X$, and let $B\subseteq X$ such that $|B|>\VC(C)$. It holds that
$\Pi_C(B)\leq \left(\frac{e|B|}{\VC(C)}\right)^{\VC(C)}$.
\end{thm}

%Observe that, as $\Pi_C(B) \leq |C|$, a set $B$ can be shattered only if $|B|\leq \log |C|$ and hence $\VC(C)\leq \log |C|$.

\subsection{VC Bounds}

Classical results in computational learning theory state that a sample of size $\Theta(\VC(C))$ is both necessary and sufficient for the PAC learning of a concept class $C$.
The following two theorems give upper and lower bounds on the sample complexity.

\begin{thm}[\cite{EHKV}]\label{thm:VClower}
For any $(\alpha,\beta{<}\frac{1}{2},n,m)$-SSL for a class $C$ it holds that $m\geq\frac{\VC(C)-1}{16\alpha}$.
\end{thm}

\begin{thm}[Generalization Bound~\cite{VC,BEHW}]\label{thm:VCconsistantOld}
Let $C$ and $\mu$ be a concept class and a distribution over a domain $X$.
Let $\alpha,\beta>0$, and $m\geq\frac{8}{\alpha}(\VC(C)\ln(\frac{16}{\alpha})+\ln(\frac{2}{\beta}))$.
Fix a concept $c\in C$, and suppose that we draw a sample $S=(x_i,y_i)_{i=1}^m$, where $x_i$ are drawn i.i.d.\ from $\mu$ and $y_i=c(x_i)$. Then,
$$
\Pr\left[\exists h\in C \text{ s.t.\ } \error_{\mu}(h,c)>\alpha \; \wedge \; \error_S(h)=0 \right]\leq\beta.
$$
\end{thm}

Hence, an algorithm that takes a sample of $m=\Omega_{\alpha,\beta}(\VC(C))$ labeled examples and outputs a concept $h\in C$ that agrees with the sample is a PAC learner for $C$. The following is a simple generalization of Theorem~\ref{thm:VCconsistantOld}.

\begin{thm}[Generalization Bound]\label{thm:VCconsistant}
Let $C$ and $\mu$ be a concept class and a distribution over a domain $X$.
Let $\alpha,\beta>0$, and $m\geq\frac{48}{\alpha}\left( 10\VC(C)\log(\frac{48e}{\alpha})+\log(\frac{5}{\beta})) \right)$.
Suppose that we draw a sample $S=(x_i)_{i=1}^m$, where each $x_i$ is drawn i.i.d.\ from $\mu$. Then,
$$
\Pr\left[  \begin{array}{c}
\exists c,h\in C \text{ s.t.\ } \error_{\mu}(c,h)\geq\alpha\\
\text{and } \error_S(c,h)\leq\alpha/10
\end{array} \right]\leq\beta.
$$
\end{thm}

The above theorem generalizes Theorem~\ref{thm:VCconsistantOld} in two aspects. First, it holds simultaneously for every pair $c,h\in C$, whereas in Theorem~\ref{thm:VCconsistantOld} the target concept $c$ is fixed before generating the sample. Second, Theorem~\ref{thm:VCconsistantOld} only ensures that a hypothesis $h$ has small generalization error if $\error_S(h)=0$. In Theorem~\ref{thm:VCconsistant} on the other hand, this is guaranteed even if $\error_S(h)$ is small (but non-zero).

%Moreover, Theorem~\ref{thm:VCconsistant} remains valid even if the target concept $c$ is chosen (perhaps in an adversarial manner) {\em after} the points $(x_i)_{i=1}^m$ are fixed.
%Such an algorithm is a PAC learner for $C$ using $C$ (that is, both the target concept and the returned hypotheses are taken from the same concept class $C$), and, therefore, there always exist a hypothesis $h\in C$ with $\error_S(h)=0$ (e.g., the target concept itself).

The next theorem handles (in particular) the agnostic case, in which the concept class $C$ is unknown and the learner uses a hypotheses class $H$. In particular, given a labeled sample $S$ there may be no $h\in H$ for which $\error_S(h)$ is small.

\begin{thm}[Agnostic Bound \cite{Anthony2009,Anthony93}]\label{thm:generalization}
Let $H$ and $\mu$ be a concept class and a distribution over a domain $X$,
and let $f:X\rightarrow\{0,1\}$ be some concept, not necessarily in $H$.
For a sample $S=(x_i,f(x_i))_{i=1}^m$ where $m\geq\frac{50 \VC(H)}{\alpha^2}\ln(\frac{1}{\alpha\beta})$
and each $x_i$ is drawn i.i.d.\ from $\mu$, it holds that
$$\Pr\Big[\forall \; h\in H,\;\; \big|\error_\mu(h,f)-\error_S(h)\big|\leq\alpha\Big]\geq1-\beta.$$  
\end{thm}

Notice that the sample size in Theorem~\ref{thm:VCconsistant} is smaller than the sample size in Theorem~\ref{thm:generalization}, where, basically, the former is proportional to $\frac{1}{\alpha}$ and the latter is proportional to $\frac{1}{\alpha^2}$.

\end{document}